\documentclass[11pt]{article}
\pdfoutput=1
\usepackage{graphicx}
\usepackage{wrapfig}
\usepackage[utf8]{inputenc} 
\usepackage[T1]{fontenc}    
\usepackage[colorlinks,
            linkcolor=red,
            anchorcolor=blue,
            citecolor=blue
            ]{hyperref}
\usepackage{url}            
\usepackage{booktabs}       
\usepackage{amsfonts}       
\usepackage{nicefrac}       
\usepackage{microtype}      
\usepackage{xcolor,colortbl}         
\usepackage{wrapfig}
\usepackage{caption}
\usepackage{tabularx}
\usepackage{natbib}
\usepackage{microtype}
\usepackage{setspace}
\usepackage{floatrow}
\usepackage{fullpage}

\usepackage{amsmath}
\usepackage{amssymb}
\usepackage{amsthm}
\usepackage{amsfonts}
\usepackage{mathtools}
\usepackage{bm}
\usepackage{bbm}
\usepackage{stmaryrd}
\usepackage{mathrsfs}

\usepackage{amsmath,amsfonts,bm}

\def\eqref#1{Eqn.~\ref{#1}}

\def\eqref#1{equation~\ref{#1}}

\def\1{\bm{1}}

\def\va{{\bm{a}}}

\DeclareMathAlphabet{\mathsfit}{\encodingdefault}{\sfdefault}{m}{sl}
\SetMathAlphabet{\mathsfit}{bold}{\encodingdefault}{\sfdefault}{bx}{n}

\def\sR{{\mathbb{R}}}

\renewcommand{\emph}[1]{\textit{#1}}

\usepackage[all]{nowidow}
\usepackage{adjustbox}
\usepackage{threeparttable}
\usepackage{etoolbox}
\appto\TPTnoteSettings{\footnotesize}
\usepackage{cellspace}
\usepackage{booktabs}
\usepackage{multirow}
\usepackage{subcaption}
\usepackage{wrapfig}
\usepackage{floatrow}
\newfloatcommand{capbtabbox}{table}[][\FBwidth]
\usepackage[ruled, linesnumbered, noend, vlined]{algorithm2e}
\usepackage{soul}
\usepackage{tikz}
\usepackage{pgfplots}
\usetikzlibrary{fit}
\usetikzlibrary{calc,shapes}
\usetikzlibrary{decorations.pathmorphing} 
\usetikzlibrary{fit}					
\usetikzlibrary{backgrounds}	
\usetikzlibrary{pgfplots.groupplots}
\usepackage{xargs}
\usepackage[capitalize]{cleveref}
\crefalias{AlgoLine}{line}%
\makeatletter
\let\cref@old@stepcounter\stepcounter
\def\stepcounter#1{%
  \cref@old@stepcounter{#1}%
  \cref@constructprefix{#1}{\cref@result}%
  \@ifundefined{cref@#1@alias}%
    {\def\@tempa{#1}}%
    {\def\@tempa{\csname cref@#1@alias\endcsname}}%
  \protected@edef\cref@currentlabel{%
    [\@tempa][\arabic{#1}][\cref@result]%
    \csname p@#1\endcsname\csname the#1\endcsname}}
\makeatother
\usepackage[shortlabels,inline]{enumitem}
\newtheoremstyle{break}
  {}
  {}
  {\itshape}
  {}
  {\bfseries}
  {}
  {\newline}
  {}
\newtheorem{definition}{Definition}
\newtheorem{assumption}{Assumption}
\newtheorem{theorem}{Theorem}
\newtheorem{proposition}[theorem]{Proposition}
\newtheorem*{proposition-informal}{Informal Proposition}
\newtheorem{lemma}[theorem]{Lemma}
\newenvironment{prevproof}[1]{\noindent {\em {Proof of \cref{#1}:}}}{\hfill $\square$\vskip \belowdisplayskip}
\newtheorem{innercustomgeneric}{\customgenericname}
\providecommand{\customgenericname}{}
\newcommand{\newcustomtheorem}[2]{%
  \newenvironment{#1}[1]
  {%
   \renewcommand\customgenericname{#2}%
   \renewcommand\theinnercustomgeneric{##1}%
   \innercustomgeneric
  }
  {\endinnercustomgeneric}
}
\newcustomtheorem{customtheorem}{Theorem}
\newcustomtheorem{customlemma}{Lemma}
\newcustomtheorem{customcorollary}{Corollary}

\usepackage{colortbl}
\pgfplotsset{compat=1.17}

\newcommand{\para}[1]{\textbf{#1}\quad}
\newcommand{\diag}[1]{\mathsf{diag}(#1)}
\newcommand{\wt}[1]{\widetilde{#1}}
\newcommand{\fC}{\mathfrak{C}}
\newcommand{\concat}{\parallel}
\newcommand{\itand}{~~{\footnotesize\textsf{and}}~~}
\newcommand{\tran}{^{\mkern-1.5mu\mathsf{T}}}
\definecolor{tab:blue}{HTML}{1f77b4}
\definecolor{tab:orange}{HTML}{ff7f0e}
\definecolor{tab:purple}{HTML}{9467bd}
\usepackage{xcolor}

\definecolor{YellowGreen}{rgb}{0.6, 0.8, 0.2}

\usepackage{xspace}

\newcommand{\ouralgo}{\textsc{HCDC}\xspace}

\newlist{todolist}{itemize}{2}
\setlist[todolist]{label=$\square$}
\usepackage{pifont}

\usepackage{comment}

\title{\huge Calibrated Dataset Condensation for Faster Hyperparameter Search}

\author{
    Mucong Ding\thanks{Department of Computer Science, University of Maryland, College Park; e-mail: {\tt mcding@umd.edu}}
    \and
    Yuancheng Xu\thanks{Department of Computer Science, University of Maryland, College Park}
    \and
    Tahseen Rabbani\footnotemark[2]
    \and
    Xiaoyu Liu\footnotemark[2]
    \and
    Brian Gravelle\thanks{Laboratory for Physical Sciences, University of Maryland}
    \and
    Teresa Ranadive\footnotemark[3]
    \and
    Tai-Ching Tuan\footnotemark[3]
    \and
    Furong Huang\footnotemark[2]
}

\begin{document}
\date{}
\maketitle

\begin{abstract}
Dataset condensation can be used to reduce 
the computational cost of training multiple models on a large dataset by condensing the training dataset into a small synthetic set. 
State-of-the-art approaches rely on matching model gradients between the real and synthetic data.
However, there is no theoretical guarantee on the generalizability of the condensed data: data condensation often generalizes poorly \emph{across hyperparameters/architectures} in practice.
In this paper, we consider a different condensation objective specifically geared toward \emph{hyperparameter search}. 
We aim to generate a synthetic validation dataset so that the validation-performance rankings of models, with different hyperparameters, on the condensed and original datasets are comparable.
We propose a novel \emph{hyperparameter-calibrated dataset condensation} (\ouralgo) algorithm, which obtains the synthetic validation dataset by matching the \emph{hyperparameter gradients} computed via implicit differentiation and efficient inverse Hessian approximation.
Experiments demonstrate that the proposed framework effectively maintains the validation-performance rankings of models and speeds up hyperparameter/architecture search for tasks on both images and graphs.

\end{abstract}

\section{Introduction}
\label{sec:intro}


Deep learning has achieved great success in various fields, such as computer vision and graph related tasks.
However, the computational cost of training state-of-the-art neural networks is rapidly increasing due to growing model and dataset sizes. 
Moreover, designing deep learning models usually requires training numerous models on the same data to obtain the optimal hyperparameters and architecture~\citep{elsken2019neural}, posing significant computational challenges. 
Thus, reducing the computational cost of repeatedly training on the same dataset is crucial. 
We address this problem from a data-efficiency perspective and consider the following question: how can one reduce the training data size for faster \emph{hyperparameter search/optimization} with minimal performance loss?

Recently, \textit{dataset distillation/condensation}~\citep{wang2018dataset} is proposed as an effective way to reduce sample size.
This approach involves producing a small \emph{synthetic} dataset to replace the original larger one, so that the test performance of the model trained on the synthetic set is comparable to that trained on the original. 
Despite the state-of-the-art performance achieved by recent dataset condensation methods when used to train a single pre-specified model, it remains challenging to utilize such methods effectively for hyperparameter search.  
Current dataset condensation methods perform poorly when applied to neural architecture search (NAS)~\citep{elsken2019neural} and when used to train deep networks beyond the pre-specified architecture~\citep{cui2022dc}.
Moreover, there is little or even a negative correlation between the performance of models trained on the synthetic vs.~the full dataset, across architectures: often, one architecture achieves higher validation accuracy when trained on the original data relative to a second architecture, but obtains lower {validation} accuracy than the second when trained on the synthetic data. 
Since architecture performance ranking is not preserved when the original data is condensed, current data condensation methods are inadequate for NAS.
This issue stems from the fact that existing condensation methods are designed on top of a single pre-specified model, and thus the condensed data may overfit this model. 

We ask: 
\emph{is it possible to preserve the architecture/hyperparameter search outcome when the original data is replaced by the condensed data?}

\begin{figure}[!htbp]
\centering
\includegraphics[width=\textwidth, clip]{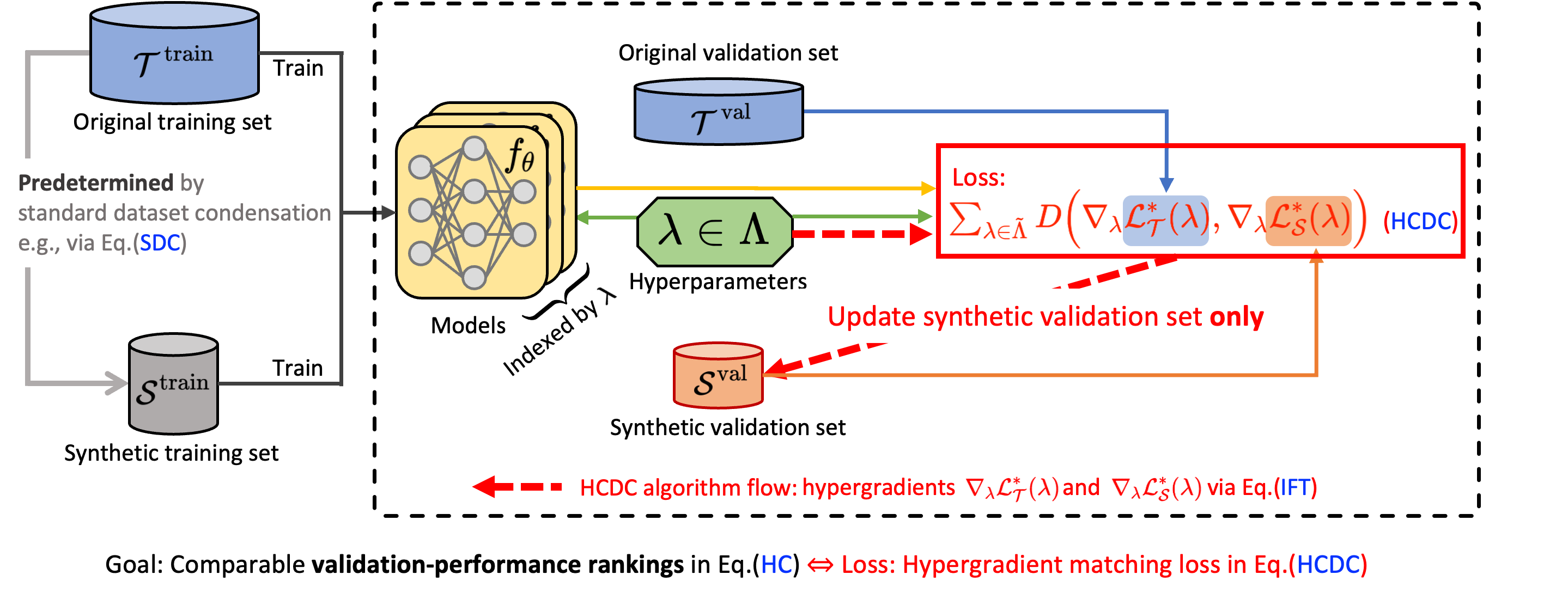}
\caption{Hyperparameter Calibrated Dataset Condensation (HCDC) aims to find a small \textcolor{orange}{validation} \textcolor{orange}{dataset} such that the validation-performance rankings of the models with different hyperparameters are comparable to the large original dataset's. Our method realizes this goal \big(\cref{eq:hc}\big) by learning the synthetic validation set to match the hypergradients w.r.t the hyperparameters \big(\cref{eq:hcdc} in the \textcolor{red}{``{Loss}'' box}\big).   
\textit{Our contribution is depicted within the big black dashed box: the algorithm flow is indicated through the \textcolor{red}{red dashed arrows}.}
The synthetic training set is predetermined by any standard dataset condensation (SDC) methods (e.g.,~\cref{eq:sdc-bl}).
The synthetic training and validation datasets obtained can later be used for hyperparameter search using only a fraction of the original computational load. A more detailed diagram is depicted in \cref{fig:schematic-appendix} in \cref{apd:diagram}.
}
\label{fig:schematic}
\end{figure}

To answer this question, we reformulate the dataset condensation problem using a hyperparameter optimization (HPO) framework~\citep{feurer2019hyperparameter}, with the goal of preserving architecture/hyperparameter search outcomes over \textit{multiple} architectures/hyperparameters, just as standard dataset condensation preserves generalization performance results for a \textit{single pre-specified} architecture.  
This is illustrated in \cref{fig:schematic}'s ``Goal'' box.
However, solving the resulting nested optimization problem is tremendously difficult.
Therefore, we consider an alternative objective and show that architecture performance ranking preservation 
is equivalent to aligning the \emph{hyperparameter gradients} (or \emph{hypergradients} for short), of this objective, in the context of dataset condensation. 
This is illustrated as the ``Loss'' box in \cref{fig:schematic}.
Thus, we propose \emph{hyperparameter calibrated dataset condensation (\ouralgo)}, a novel condensation method that preserves hyperparameter performance rankings 
by aligning the hypergradients computed using the condensed data to those computed using the original dataset, see~\cref{fig:schematic}.

Our implementation of~\ouralgo is efficient and scales linearly with respect to the size of hyperparameter search space.
Moreover, hypergradients are efficiently computed with constant memory overhead, using the implicit function theorem (IFT) and the Neumann series approximation of an inverse Hessian~\citep{lorraine2020optimizing}.
We also specifically consider how to apply~\ouralgo to the practical architecture search spaces for image and graph datasets.

Experiments demonstrate that our proposed~\ouralgo algorithm drastically increases the correlation between the architecture rankings of models trained on the condensed dataset and those trained on the original dataset, for both image and graph data.
Additionally, the test performance of the highest ranked architecture determined by the condensed dataset is comparable to that of the true optimal architecture determined by the original dataset.
Thus, \ouralgo can enable faster hyperparameter search and obtain high performance accuracy by choosing the highest ranked hyperparameters, while the other condensation and coreset methods cannot.
We also demonstrate that condensed datasets obtained with~\ouralgo are compatible with off-the-shelf architecture search algorithms with or without parameter sharing.

We summarize our contributions as follows:
\textbf{(1)} We study the data condensation problem for hyperparameter search and show that 
performance ranking preservation is equivalent to hypergradient alignment in this context.
\textbf{(2)} We propose \ouralgo, which synthesizes condensed data by aligning the hypergradients of the objectives associated with the condensed and original datasets for faster hyperparameter search.
\textbf{(3)} We present experiments, for which \ouralgo drastically reduces the search time and complexity of off-the-shelf NAS algorithms, for both image and graph data, while preserving the search outcome with high accuracy.

\section{Standard Dataset Condensation}
\label{sec:preliminaries}

Consider a classification problem {where} the original dataset $\mathcal{T}^{\mathrm{train}} = \{(x_i,y_i)\}_{i=1}^{n}$ consists of $n$ (input, label) pairs sampled from the original data distribution $P_{\mathcal{D}}$.  
To simplify notation, we replace $\mathcal{T}^{\mathrm{train}}$ with $\mathcal{T}$ when the context is clear.
The classification task goal is to train a function $f_{\theta}$ (e.g., a deep neural network), with parameter $\theta$, to correctly predict labels $y$ from input{s} $x$. {Obtaining} $f_{\theta}$ involves {optimizing an} empirical loss objective determined by $\mathcal{T}^{\mathrm{train}}$:
\begin{equation}
\theta^{\mathcal{T}}= \arg\min_{\theta} \mathcal{L}^{\mathrm{train}}_{\mathcal{T}}(\theta,\lambda) \text{,~where~} \mathcal{L}^{\mathrm{train}}_{\mathcal{T}}(\theta,\lambda) \coloneqq \frac{1}{|\mathcal{T}^{\mathrm{train}}|}\sum_{(x,y)\in \mathcal{T}^{\mathrm{train}}}l(f_{\theta}(x),y,\lambda),
\end{equation}
where $\lambda$ {denotes} the {model} hyperparameter (e.g., the neural network architecture {that characterizes} $f_{\theta}$), and $l(\cdot,\cdot,\cdot)$ is a task-specific loss function that depends {on $\lambda$.}

{Dataset condensation involves generating} a small set of {$c\ll n$} synthesized samples $\mathcal{S} = \{x_i',y_i'\}_{i=1}^{c}$, with which to replace the original {training dataset $\mathcal{T}$}. {Using} the condensed dataset $\mathcal{S}$, one can {obtain} $f_{\theta}$ {with parameter} $\theta = \theta^{\mathcal{S}} = \arg\min_{\theta} \mathcal{L}^{\mathrm{train}}_{\mathcal{S}}(\theta,\lambda)$, where $\mathcal{L}^{\mathrm{train}}_{\mathcal{S}} = \frac{1}{|\mathcal{S}|}\sum_{(x,y)\in \mathcal{S}}l(f_{\theta}(x),y,\lambda)$. 
{The goal is for} the generalization performance of the model {$f_{\theta^{\mathcal{S}}}$ obtained using} the condensed data to 
{approximate that of} $f_{\theta^{\mathcal{T}}}$, i.e., $\mathbb{E}_{(x,y)\sim P_{\mathcal{D}}} [l(f_{\theta^{\mathcal{T}}}(x),y,\lambda)] \approx \mathbb{E}_{(x,y)\sim P_{\mathcal{D}}}  [l(f_{\theta^{\mathcal{S}}}(x),y,\lambda)]$.

Next, we review the bi-level optimization formulation of the \emph{standard dataset condensation} (SDC)~\citep{wang2018dataset} and one of its efficient solutions using gradient matching~\citep{zhao2020dataset}.

\para{SDC's objective.} 
By posing the optimal parameters $\theta^{\mathcal{S}}(\mathcal{S})$ as a function of the condensed dataset $\mathcal{S}$, SDC can be formulated as a bi-level optimization problem as follows,
\begin{equation}
\label{eq:sdc-bl}
\mathcal{S}^* = \arg\min_{\mathcal{S}}\mathcal{L}^{\mathrm{train}}_{\mathcal{T}}(\theta^{\mathcal{S}}(\mathcal{S}),\lambda) \mathrm{,~s.t.~} \theta^{\mathcal{S}}(\mathcal{S}) \coloneqq \arg\min_{\theta}\mathcal{L}^{\mathrm{train}}_{\mathcal{S}}(\theta,\lambda).
\tag{SDC}
\end{equation}
In other words, the optimization problem in \cref{eq:sdc-bl} aims to ﬁnd the optimal synthetic dataset $\mathcal{S}$ such that the model $\theta^{\mathcal{S}}(\mathcal{S})$ trained on it minimizes the training loss over the original data $\mathcal{T}^{\mathrm{train}}$. 
However, directly solving the optimization problem in \cref{eq:sdc-bl} is difficult since it involves a nested-loop optimization and solving the inner loop for $\theta^{\mathcal{S}}(\mathcal{S})$ at each iteration requires unrolling the recursive computation graph for $\mathcal{S}$ over multiple optimization steps for $\theta$~\citep{domke2012generic}, which is computationally expensive.

\para{SDC in a gradient matching formulation.} \cite{zhao2020dataset}~alleviate this computational issue by introducing a \emph{gradient matching} (GM) formulation.
Firstly, they formulate the condensation objective as not only achieves comparable generalization performance to $\theta^{\mathcal{T}}$ but also converges to a similar solution in the parameter space, i.e., $\theta^{\mathcal{S}}(\mathcal{S},\theta_0) \approx \theta_{\mathcal{T}}(\theta_0)$, where $\theta_0$ indicates the initialization. 
The resulting formulation is still a bilevel optimization but can be simplified via several approximations.

\textit{(1)} $\theta^{\mathcal{S}}(\mathcal{S},\theta_0)$ is approximated by the output of a series of gradient-descent updates, $\theta^{\mathcal{S}}(\mathcal{S},\theta_0)\approx\theta^{\mathcal{S}}_{t+1}\leftarrow\theta^{\mathcal{S}}_{t}-\eta\nabla_\theta\mathcal{L}^{\mathrm{train}}_{\mathcal{S}}(\theta^{\mathcal{S}}_t,\lambda)$. 
In addition, \cite{zhao2020dataset} propose to match $\theta^{\mathcal{S}}_{t+1}$ with incompletely optimized $\theta^{\mathcal{T}}_{t+1}$ at each iteration $t$.
Consequently, the dataset condensation objective is now $\mathcal{S}^*=\arg\min_{\mathcal{S}}\mathbb{E}_{\theta_0\sim P_{\theta_0}}[\sum_{t=0}^{T-1}D(\theta^{\mathcal{S}}_t,\theta^{\mathcal{T}}_t)]$.

\textit{(2)} If we assume $\theta^{\mathcal{S}}_t$ can always track $\theta^{\mathcal{T}}_t$ (i.e., $\theta^{\mathcal{S}}_t\approx\theta^{\mathcal{T}}_t$) from the initialization $\theta_0$ up to iteration $t$, then we can replace $D(\theta^{\mathcal{S}}_{t+1}, \theta^{\mathcal{T}}_{t+1})$ by $D(\nabla_\theta\mathcal{L}^{\mathrm{train}}_{\mathcal{S}}(\theta^{\mathcal{S}}_t,\lambda),\nabla_\theta\mathcal{L}^{\mathrm{train}}_{\mathcal{T}}(\theta^{\mathcal{T}}_t,\lambda))$.
The final objective for the GM formulation is,
\begin{equation}
\label{eq:sdc-gm}
\min_{\mathcal{S}}\mathbb{E}_{\theta_0\sim P_{\theta_0}}\Big[\sum_{t=0}^{T-1}D\Big(
\nabla_\theta\mathcal{L}^{\mathrm{train}}_{\mathcal{S}}(\theta^{\mathcal{S}}_t,\lambda),
\nabla_\theta\mathcal{L}^{\mathrm{train}}_{\mathcal{T}}(\theta^{\mathcal{S}}_t,\lambda)\Big)\Big].
\end{equation}


\paragraph{Challenge of varying hyperparameter $\lambda$.}
In the formulation of the SDC, the condensed data $\mathcal{S}$ is learned with a fixed hyperparameter $\lambda$, e.g., a pre-specified neural network architecture. 
As a result, the condensed data trained with SDC's objective performs poorly on hyperparameter search~\citep{cui2022dc}, which requires the performance of models under varying hyperparameters to behave consistently on the original and condensed dataset. 
In the following, we tackle this issue by reformulating the dataset condensation problem under the hyperparameter optimization framework. 
\section{Hyperparameter Calibrated Dataset Condensation}
\label{sec:method}

In this section, we would like to develop a condensation method specifically for preserving the outcome of \emph{hyperparameter optimization} (HPO)  on the condensed dataset across different architectures/hyperparameters for faster hyperparameter search. 
This requires dealing with varying choices of hyperparameters so that the relative performances of different hyperparameters on the condensed and original datasets are consistent. 
We first formulate the data condensation for hyperparameter search in the HPO framework below and then propose the hyperparameter calibrated dataset condensation framework in Section~\ref{ssec: hcdc_method} by using the equivalence relationship between preserving the performance ranking and the hypergradient alignment.

\para{HPO's objective.} 
Given $\mathcal{T} = \mathcal{T}^{\mathrm{train}} \bigcup \mathcal{T}^{\mathrm{val}} \bigcup \mathcal{T}^{\mathrm{test}}$, HPO aims to find the optimal hyperparameter $\lambda^\mathcal{T}$ that minimizes the validation loss of the model optimized on the training dataset $\mathcal{T}^{\mathrm{train}}$ with hyperparameter $\lambda^\mathcal{T}$, i.e.,
\begin{equation}
\label{eq:hpo}
\lambda^\mathcal{T}=\arg\min_{\lambda\in\Lambda}\mathcal{L}^*_{\mathcal{T}}(\lambda) \text{,~where~} \mathcal{L}^*_{\mathcal{T}}(\lambda) \coloneqq \mathcal{L}^{\mathrm{val}}_{\mathcal{T}}(\theta^{\mathcal{T}}(\lambda),\lambda) \text{~and~} \theta^{\mathcal{T}}(\lambda)  \coloneqq\arg\min_{\theta}\mathcal{L}^{\mathrm{train}}_{\mathcal{T}}(\theta,\lambda).
\tag{HPO}
\end{equation}
Here $\mathcal{L}^{\mathrm{val}}_{\mathcal{T}}(\theta,\lambda) \coloneqq \frac{1}{|\mathcal{T}^{\mathrm{val}}|}\sum_{(x,y)\in \mathcal{T}^{\mathrm{val}}}l(f_{\theta}(x),y,\lambda)$. 
HPO is a bi-level optimization where both the optimal parameter $\theta^{\mathcal{T}}(\lambda)$ and the optimized validation loss $\mathcal{L}^*_{\mathcal{T}}(\lambda)$ are viewed as a function of the hyperparameter $\lambda$.

\paragraph{Dataset condensation for HPO.} 
We would like to synthesize a condensed training dataset $\mathcal{S}^{\mathrm{train}}$ and a condensed validation dataset $\mathcal{S}^{\mathrm{val}}$ to replace the original $\mathcal{T}^{\mathrm{train}}$ and $\mathcal{T}^{\mathrm{val}}$ for hyperparameter search. 
Denote the synthetic dataset as $\mathcal{S} = \mathcal{S}^{\mathrm{train}} \bigcup \mathcal{S}^{\mathrm{val}}$. 
Similar to \cref{eq:hpo}, the optimal hyperparameter $\lambda^\mathcal{S}$ is defined for a given dataset $\mathcal{S}$.  
Naively, one can formulate such a problem as finding the condensed dataset $\mathcal{S}$ to minimize the validation loss on the original dataset $\mathcal{T}^{\mathrm{val}}$ as follows, which is an optimization problem similar to the standard dataset condensation in \cref{eq:sdc-bl}:
\begin{equation}
\label{eq:hcdc-bl}
\mathcal{S}^* = \arg\min_{\mathcal{S}}\mathcal{L}^*_{\mathcal{T}}\big(\lambda^\mathcal{S}(\mathcal{S})\big) \quad\mathrm{s.t.}\quad \lambda^\mathcal{S}(\mathcal{S}) \coloneqq \arg\min_{\lambda\in\Lambda}\mathcal{L}^*_{\mathcal{S}}(\lambda),    
\end{equation}
where the optimized validation losses $\mathcal{L}^*_{\mathcal{T}}(\cdot)$ and $\mathcal{L}^*_{\mathcal{S}}(\cdot)$ are defined following~\cref{eq:hpo}.

However, \textbf{two challenges} exist for such a formulation. 
\textbf{Challenge (1)}: \cref{eq:hcdc-bl} is a nested optimization (for dataset condensation) over another nested optimization (for HPO), which is computationally expensive.
\textbf{Challenge (2)}: the search space $\Lambda$ of the hyperparameters can be complicated. 
In contrast to parameter optimization, where the search space is usually assumed to be the continuous and unbounded Euclidean space, the search space of the hyperparameters can be compositions of discrete and continuous spaces. 
Having such discrete components in the search space poses challenges for gradient-based optimization methods.

To address \textbf{Challenge (1)}, we propose an alternative objective based on the alignment of hypergradients that can be computed efficiently in \cref{ssec: hcdc_method}. 
For \textbf{Challenge (2)}, we construct the extended search space in \cref{sec:algorithm}.

\section{Hyperparameter Calibration via Hypergradient Alignment}
\label{ssec: hcdc_method}

In this section, we introduce Hyperparameter-Calibrated Dataset Condensation (HCDC), a novel condensation method designed to align \textit{hyperparameter gradients} -- referred to as \textit{hypergradients} -- thus preserving the validation performance ranking of various hyperparameters.

\para{Hyperparameter calibration.}
To tackle the computational challenges inherent in hyperparameter optimization (HPO) as expressed in \cref{eq:hcdc-bl}, we propose an efficient yet sufficient alternative. Rather than directly solving the HPO problem, we aim to identify a condensed dataset that maintains the outcomes of HPO on the hyperparameter set $\Lambda$. We refer to this process as hyperparameter calibration, formally defined as follows.
\begin{definition}[Hyperparameter Calibration]
\label{def:hc}
Given original dataset $\mathcal{T}$, generic model $f^{\lambda}_{\theta}$, and hyperparameter search space $\Lambda$, we say a condensed dataset $\mathcal{S}$ is hyperparameter calibrated, if for any $\lambda_1\neq \lambda_2\in\Lambda$, it holds that,
\begin{equation}
\label{eq:hc}
\Big(\mathcal{L}_{\mathcal{T}}^*(\lambda_1)-\mathcal{L}_{\mathcal{T}}^*(\lambda_2)\Big)\Big(\mathcal{L}_{\mathcal{S}}^*(\lambda_1)-\mathcal{L}_{\mathcal{S}}^*(\lambda_2)\Big)>0
\tag{HC}
\end{equation}
In other words, changes of the optimized validation loss on $\mathcal{T}$ and $\mathcal{S}$ {always have the same sign}, between any pairs of hyperparameters $\lambda_1\neq\lambda_2$.
\end{definition}

It is evident that if hyperparameter calibration (HC) is satisfied, the outcomes of HPO on both the original and condensed datasets will be identical. Consequently, our objective shifts to \emph{ensuring hyperparameter calibration across all pairs of hyperparameters}.

\para{\ouralgo: hypergradient alignment objective for dataset condensation.}
To move forward, we make the assumption that there exists a continuous extension of the search space.
Specifically, the (potentially discrete) search space $\Lambda$ can be extended to a compact and connected set $\tilde{\Lambda}\supset\Lambda$.
Within this extended set, 
we define a continuation of the generic model $f^{\lambda}_{\theta}$ 
such that $f^{\lambda}_{\theta}$ is differentiable anywhere in $\tilde{\Lambda}$.
In~\cref{sec:algorithm}, we will elaborate on how to construct such an extended search space $\tilde{\Lambda}$.

To establish a new objective for hyperparameter calibration, consider the case when $\lambda_1$ is in the neighborhood of $\lambda_2$, denoted as $\lambda_1\in B_r(\lambda_2)$ for some $r>0$. In this situation, the change in validation loss can be approximated up to first-order by the hypergradients, as follows: $\mathcal{L}_{\mathcal{T}}^*(\lambda_1)-\mathcal{L}_{\mathcal{T}}^*(\lambda_2)\approx
\langle\nabla_\lambda\mathcal{L}_{\mathcal{T}}^*(\lambda_2), \Delta\lambda\rangle.$
Here, $\Delta\lambda=\lambda_1 - \lambda_2$ with $r\geq\|\Delta\lambda\|_2\to0^{+}$. 
Analogously, for the synthetic dataset we have: $\mathcal{L}_{\mathcal{S}}^*(\lambda_1)-\mathcal{L}_{\mathcal{S}}^*(\lambda_2)\approx \langle\nabla_\lambda\mathcal{L}_{\mathcal{S}}^*(\lambda_2), \Delta\lambda\rangle$.
Hence, the hyperparameter calibration condition simplifies to $ \langle\nabla_\lambda\mathcal{L}_{\mathcal{T}}^*(\lambda_2), \Delta\lambda\rangle \cdot \langle\nabla_\lambda\mathcal{L}_{\mathcal{S}}^*(\lambda_2), \Delta\lambda\rangle > 0$. 
Further simplification leads to 
$\nabla_\lambda\mathcal{L}_{\mathcal{T}}^*(\lambda)\parallel\nabla_\lambda\mathcal{L}_{\mathcal{S}}^*(\lambda)$, indicating \textit{alignment} of the two hypergradient vectors.
We formally define this hypergradient alignment and establish its equivalence to hyperparameter calibration.


\begin{definition}[Hypergradient Alignment]\label{def:hg-align} We say hypergradients are aligned in an \emph{extended} search space $\tilde{\Lambda}$, if for any $\lambda\in\tilde{\Lambda}$, it holds that $\nabla_\lambda\mathcal{L}_{\mathcal{T}}^*(\lambda)\parallel\nabla_\lambda\mathcal{L}_{\mathcal{S}}^*(\lambda)$, i.e., $D_c(\nabla_\lambda\mathcal{L}_{\mathcal{T}}^*(\lambda),\nabla_\lambda\mathcal{L}_{\mathcal{S}}^*(\lambda))=0$, where $D_c(\cdot,\cdot)=1-\cos(\cdot,\cdot)$ represents the cosine distance.
\end{definition}

\begin{theorem}[\textbf{Equivalence between
Hypergradient Alignment and Hyperparameter Calibration}]
\label{thm:equi-hga-hc}
Hypergradient alignment (\cref{def:hg-align}) is equivalent to hyperparameter calibration (\cref{def:hc}) on a connected and compact set, e.g., the extended search space $\tilde{\Lambda}$.
\end{theorem}


The implication is straightforward: if hyperparameter calibration holds in $\tilde{\Lambda}$, it also holds in $\Lambda$.
According to \cref{thm:equi-hga-hc}, achieving hypergradient alignment in $\tilde{\Lambda}$ is sufficient to ensure hyperparameter calibration in $\Lambda$. Therefore, the integrity of the HPO outcome over $\Lambda$ is maintained.



Consequently, the essence of our hyperparameter calibrated dataset condensation (\ouralgo) is to align/match the hypergradients calculated on both the original and condensed datasets within the extended search space $\tilde{\Lambda}$:
\begin{equation}
\label{eq:hcdc}
\mathcal{S}^*=\arg\min_{\mathcal{S}}
\sum_{\lambda\in\tilde{\Lambda}} D_c\Big(\nabla_\lambda\mathcal{L}_{\mathcal{T}}^{\mathrm{val}}\big(\theta^{\mathcal{T}}(\lambda),\lambda\big),\nabla_\lambda\mathcal{L}_{\mathcal{S}}^{\mathrm{val}}\big(\theta^{\mathcal{S}}(\lambda),\lambda\big)\Big),
\tag{{HCDC}}
\end{equation}
where the cosine distance $D_c(\cdot,\cdot)=1-\cos(\cdot,\cdot)$ is used without loss of generality.

\section{Implementations of \ouralgo}
\label{sec:algorithm}


In this section, we focus on implementing the hyperparameter calibrated dataset condensation (\ouralgo) algorithm.
We address two primary challenges: \textbf{(1)} efficient approximate computation of hyperparameter gradients, often called hypergradients,  using implicit differentiation techniques; and \textbf{(2)} the efficient formation of the extended search space $\tilde{\Lambda}$.
The complete pseudocode for HCDC will be provided at the end of this section.

\subsection{Efficient Evaluation of Hypergradients}
\label{ssec:hypergrad-evaluation}

The efficient computation of hypergradients is  well-addressed in existing literature (see~\cref{sec:related}). In our HCDC implementation, we utilize the implicit function theorem (IFT) and the Neumann series approximation for inverse Hessians, as proposed by~\citet{lorraine2020optimizing}.

\para{Computing hypergradients via IFT.}
The hypergradients are the gradients of the optimized validation loss $\mathcal{L}^*_{\mathcal{T}}(\lambda)=\mathcal{L}^{\mathrm{val}}_{\mathcal{T}}(\theta^{\mathcal{T}}(\lambda), \lambda)$ with respect to the hyperparameters $\lambda$; see~\cref{apd:hypergradients} for further details.
The implicit function theorem (IFT) provides an efficient approximation to compute the hypergradients $\nabla_\lambda\mathcal{L}_{\mathcal{T}}^*(\lambda)$ and $\nabla_\lambda\mathcal{L}_{\mathcal{S}}^*(\lambda)$.
\begin{equation}
\label{eq:ift}
\nabla_\lambda\mathcal{L}_{\mathcal{T}}^*(\lambda) \approx -\big[\frac{\partial^2\mathcal{L}_{\mathcal{T}}^{\mathrm{train}}(\theta,\lambda)}{\partial\lambda \partial\theta^T}\big]\big[\frac{\partial^2\mathcal{L}_{\mathcal{T}}^{\mathrm{train}}(\theta,\lambda)}{\partial\theta \partial\theta^T}\big]^{-1} \nabla_\theta\mathcal{L}_{\mathcal{T}}^{\mathrm{val}}(\theta,\lambda)+\nabla_\lambda\mathcal{L}_{\mathcal{T}}^{\mathrm{val}}(\theta,\lambda),
\tag{IFT}
\end{equation}
where we consider the direct gradient $\nabla_\lambda\mathcal{L}_{\mathcal{T}}^{\mathrm{val}}(\theta,\lambda)$ is $0$, since in most cases the hyperparameter $\lambda$ only affects the validation loss $\mathcal{L}_{\mathcal{T}}^{\mathrm{val}}(\theta,\lambda)$ through the model function $f_{\theta,\lambda}$. 
The first term consists of the mixed partials $\big[\frac{\partial^2\mathcal{L}_{\mathcal{T}}^{\mathrm{train}}(\theta,\lambda)}{\partial\lambda \partial\theta^T}\big]$, the inverse Hessian $\big[\frac{\partial^2\mathcal{L}_{\mathcal{T}}^{\mathrm{train}}(\theta,\lambda)}{\partial\theta \partial\theta^T}\big]^{-1}$, and the validation gradients $\nabla_\theta\mathcal{L}_{\mathcal{T}}^{\mathrm{val}}(\theta,\lambda)$. While the other parts can be calculated efficiently through a single back-propagation, approximating the inverse Hessian is required.
\citet{lorraine2020optimizing} propose a stable, tractable, and efficient Neumann series approximation of the inverse Hessian as follows:
\begin{equation*}
\textstyle
\big[\frac{\partial^2\mathcal{L}_{\mathcal{T}}^{\mathrm{train}}(\theta,\lambda)}{\partial\theta \partial\theta^T}\big]^{-1}=\lim _{i \rightarrow \infty} \sum_{j=0}^{i}\big[I-\frac{\partial^2\mathcal{L}_{\mathcal{T}}^{\mathrm{train}}(\theta,\lambda)}{\partial\theta \partial\theta^T}\big]^{j},
\end{equation*}
which requires only constant memory.
When combined with~\cref{eq:ift}, the approximated hypergradients can be evaluated by employing efficient 
vector-Jacobian products~\citep{lorraine2020optimizing}.

\para{Optimizing hypergradient alignment loss in~\cref{eq:hcdc}.}
To optimize the objective defined in \ouralgo (\cref{eq:hcdc}), we learn the synthetic validation set $\mathcal{S}^{\mathrm{val}}$ from scratch.
This is crucial as the hypergradients with respect to the validation losses in~\cref{eq:hcdc}, are significantly influenced by the synthetic validation examples, which are free learnable parameters during the condensation process.
In contrast, we maintain the synthetic training set $\mathcal{S}^{\mathrm{train}}$ as fixed.
For generating $\mathcal{S}^{\mathrm{train}}$, we employ the standard dataset condensation (SDC) algorithm, as described in \cref{eq:sdc-gm}.
To optimize the synthetic validation set $\mathcal{S}^{\mathrm{val}}$ with respect to the hyper-gradient loss in~\cref{eq:hcdc}, we compute the gradients of $\nabla_\theta\mathcal{L}_{\mathcal{S}}^{\mathrm{val}}(\theta,\lambda)$ and $ \nabla_\lambda\mathcal{L}_{\mathcal{S}}^{\mathrm{val}}(\theta,\lambda)$  w.r.t. $\mathcal{S}^{\mathrm{val}}$.
This is handled using an additional back-propagation step, akin to the one in SDC that calculates the gradients of $\nabla_\theta\mathcal{L}_{\mathcal{S}}^{\mathrm{train}}(\theta,\lambda)$ w.r.t $\mathcal{S}^{\mathrm{train}}$.

\subsection{Efficient Design of Extended Search Space}
\label{ssec:efficient-search-space}

\ouralgo's objective (\cref{eq:hcdc}) necessitates the alignment of hypergradients across all hyperparameters $\lambda$'s in an extended space $\tilde{\Lambda}$.
This space is a compact and connected superset of the original search space $\Lambda$.
For practical implementation, we evaluate the hypergradient matching loss using a subset of $\lambda$ values randomly sampled from $\tilde{\Lambda}$.
To enhance HCDC's efficiency within a predefined search space $\lambda$, our goal is to minimally extend this space to $\tilde{\Lambda}$ for sampling.

In the case of continuous hyperparameters, $\Lambda$ is generally both compact and connected, rending $\tilde{\Lambda}$ identical to $\Lambda$.
For discrete search spaces $\Lambda$ consisting of $p$ candidate hyperparameters,
we propose a linear-complexity construction for $\tilde{\Lambda}$ (in which the linearity is in terms of $p$).
Specifically, for each $i\in[p]$, we formulate an ``$i$-th HPO trajectory'', a representative path that originates from $\lambda^\mathcal{S}_{i,0}=\lambda_i\in\Lambda$ and evolves via the update rule $\lambda^\mathcal{S}_{i,t+1}\leftarrow\lambda^\mathcal{S}_{i,t}-\eta\nabla_\lambda\mathcal{L}_{\mathcal{S}}^*(\lambda^\mathcal{S}_{i,t})$, see~\cref{sapd:efficient} for details and~\cref{fig:trajectory-matching} for illustration.
We assume that all $p$ trajectories converge to the same or equivalent optima $\lambda^{\mathcal{S}}$, thus forming ``connected'' paths. 
Consequently, the extended search space $\tilde{\Lambda}$ comprises these $p$ connected trajectories, allowing us to evaluate the hypergradient matching loss along each trajectory $\{\lambda^\mathcal{S}_{i,t}\}_{t=0}^{T}$ during the iterative update of $\lambda$.


\subsection{Pseudocode}

We conclude this section by outlining the HCDC algorithm in~\cref{alg:hcdc}, assuming a discrete and finite hyperparameter search space $\Lambda$. 
In~\cref{alg:hcdc:update-hyperparameters}, we calculate the hypergradients $\nabla_\lambda\mathcal{L}^*_\mathcal{S}(\lambda)$ using~\cref{eq:ift}.
For computing the gradient $\nabla_{\mathcal{S}^\mathrm{val}} D\big(\nabla_\lambda\mathcal{L}^*_\mathcal{T}(\lambda),\nabla_\lambda\mathcal{L}^*_\mathcal{S}(\lambda)\big)$ in~\cref{alg:hcdc:update-synthetic-set}, we note that only $\nabla_\lambda\mathcal{L}^*_\mathcal{S}(\lambda)$ is depends on $\mathcal{S}^\mathrm{val}$.
Employing~\cref{eq:ift}, we find that $\nabla_\lambda\mathcal{L}^*_\mathcal{S}(\lambda)=-\big[\frac{\partial^2\mathcal{L}_{\mathcal{S}}^{\mathrm{train}}(\theta,\lambda)}{\partial\lambda \partial\theta^S}\big]\big[\frac{\partial^2\mathcal{L}_{\mathcal{S}}^{\mathrm{train}}(\theta,\lambda)}{\partial\theta \partial\theta^S}\big]^{-1}\nabla_\theta\mathcal{L}_{\mathcal{S}}^{\mathrm{val}}(\theta,\lambda)$.
Note that there are no direct gradients, as $\lambda$ influences the loss solely through the model $f^{\lambda}_{\theta}$.
Therefore, to obtain $\nabla_{\mathcal{S}^\mathrm{val}} D\big(\nabla_\lambda\mathcal{L}^*_\mathcal{T}(\lambda),\nabla_\lambda\mathcal{L}^*_\mathcal{S}(\lambda)\big)$, we simply need to compute the gradient $\nabla_{\mathcal{S}^\mathrm{val}}\nabla_\theta\mathcal{L}_{\mathcal{S}}^{\mathrm{val}}(\theta,\lambda)$ through standard back-propagation methods, since only the validation loss term $\nabla_\theta\mathcal{L}_{\mathcal{S}}^{\mathrm{val}}(\theta,\lambda)$ depends on $\mathcal{S}^\mathrm{val}$.

\begin{figure}[!htbp]
\centering
\begin{algorithm}[H]
\setstretch{0.6}
\small
\caption{\label{alg:hcdc}Hyperparameter Calibrated Dataset Condensation (HCDC)}
\DontPrintSemicolon
\KwIn{Original dataset $\mathcal{T}$, a set of NN architectures $f_{\theta}$, hyperparameter search space $\lambda\in\Lambda=\{\lambda_1,\ldots,\lambda_p\}$, predetermined condensed training data $\mathcal{S}^\mathrm{train}$ learned by standard dataset condensation (e.g., \cref{eq:sdc-gm}), randomly initialized synthetic examples $\mathcal{S}^\mathrm{val}$ of $C$ classes.}
\For{repeat $k=0,\ldots,K-1$}{
    \ForEach{hyperparameters $\lambda=\lambda_1,\ldots,\lambda_p$ in $\Lambda$}{
        Initialize model parameters $\theta\leftarrow\theta_0\sim P_{\theta_0}$\;
        \For{epoch $t=0,\ldots,T_\theta-1$}{
            Update model parameters $\theta \leftarrow \theta - \eta_\theta \nabla_\theta\mathcal{L}^\mathrm{train}_\mathcal{S}(\theta,\lambda)$\;
            \If{$t\mod T_\lambda=0$}{
                Update hyperparameters $\lambda \leftarrow \lambda - \eta_\lambda\nabla_\lambda \mathcal{L}^*_\mathcal{S}(\lambda)$\;\label{alg:hcdc:update-hyperparameters}
            }
            Update the synthetic validation set $\mathcal{S}^\mathrm{val}\leftarrow \mathcal{S}^\mathrm{val} - \eta_\mathcal{S}\nabla_{\mathcal{S}^\mathrm{val}} D\big(\nabla_\lambda\mathcal{L}^*_\mathcal{T}(\lambda),\nabla_\lambda\mathcal{L}^*_\mathcal{S}(\lambda)\big)$\;\label{alg:hcdc:update-synthetic-set}
        }
    }
}
\Return{Condensed validation set $\mathcal{S}^\mathrm{val}$.}
\end{algorithm}
\end{figure}

For a detailed complexity analysis of~\cref{alg:hcdc} and further discussions, refer to~\cref{apd:complexity-analysis}.

\section{Related Work}
\label{sec:related}

The traditional way to simplify a dataset is \textbf{coreset selection}~\citep{toneva2018empirical,paul2021deep}, where critical training data samples are chosen based on heuristics like diversity~\citep{aljundi2019gradient}, distance to the dataset cluster centers~\citep{rebuffi2017icarl,chen2010super} and forgetfulness~\citep{toneva2018empirical}. 
However, the performance of coreset selection methods is limited by the assumption of the existence of representative samples in the original data, which may not hold in practice.

To overcome this limitation, \textbf{dataset distillation/condensation}~\citep{wang2018dataset} has been proposed as a more effective way to reduce sample size.
Dataset condensation (or dataset distillation) is first proposed in~\citep{wang2018dataset} as a learning-to-learn problem by formulating the network parameters as a function of synthetic data and learning them through the network parameters to minimize the training loss over the original data. 
This approach involves producing a small \emph{synthetic} dataset to replace the original larger one, so that the test/generalization performance of the model trained on the synthetic set is comparable to that trained on the original. 
However, the nested-loop optimization precludes it from scaling up to large-scale in-the-wild datasets. \citet{zhao2020dataset} alleviate this issue by enforcing the gradients of the synthetic samples w.r.t. the network weights to approach those of the original data, which successfully alleviates the expensive unrolling of the computational graph. Based on the meta-learning formulation in~\citep{wang2018dataset}, \citet{bohdal2020flexible} and \citet{nguyen2020dataset,nguyen2021dataset} propose to simplify the inner-loop optimization of a classification model by training with ridge regression which has a closed-form solution, while \citet{such2020generative} model the synthetic data using a generative network. To improve the data efficiency of synthetic samples in the gradient-matching algorithm, \citet{zhao2021dataseta} apply differentiable Siamese augmentation, and~\citet{kim2022dataset} introduce efficient synthetic-data parametrization.


\textbf{Implicit differentiation} methods apply the implicit function theorem (IFT)~(\cref{eq:ift}) to nested-optimization problems~\citep{wang2019solving}.
\citet{lorraine2020optimizing} approximated the inverse Hessian by Neumann series, which is a stable alternative to conjugate gradients~\citep{shaban2019truncated} and scales IFT to large networks with constant memory.

\textbf{Differentiable NAS} methods, e.g., DARTS~\citep{liu2018darts} explore the possibility of transforming the discrete neural architecture space into a continuously differentiable form and further uses gradient optimization to search the neural architecture. SNAS~\citep{xie2018snas} points out that DARTS suffers from the unbounded bias issue towards its objective, and it remodels the NAS and leverages the Gumbel-softmax trick~\citep{jang2016categorical,maddison2016concrete} to learn the 
architecture parameter.

In addition, we summarize  more dataset condensation and coreset selection methods as well as graph reduction methods in~\cref{apd:extend-related}.

\begin{table*}[!htbp]
\adjustbox{max width=\textwidth}{%
{\renewcommand{\arraystretch}{1.5}%
{\Large
\begin{tabular}{cccccccccccc}
\toprule
                            & \multicolumn{1}{c}{} & \multicolumn{3}{c}{Coresets}                                                              & \multicolumn{5}{c}{Standard Condensation}                                                                                                                                                                            & Ours          & Oracle  \\ 
\multirow{-2}{*}{Dataset}   &                       & \multicolumn{1}{c}{Random}         & \multicolumn{1}{c}{K-Center}       & Herding        & \multicolumn{1}{c}{DC}             & \multicolumn{1}{c}{DSA}            & \multicolumn{1}{c}{DM}             & \multicolumn{1}{c}{KIP}            & TM                                                            & HCDC          & Optimal \\ \midrule
                            & Corr. & $-0.12\pm0.07$ & $0.19\pm0.12$  & $-0.05\pm0.08$ & $-0.21\pm0.15$ & $-0.33\pm0.09$ & $-0.10\pm0.15$ & $-0.27\pm0.15$ & \cellcolor[HTML]{FFFFFF}{\color[HTML]{333333} $-0.07\pm0.04$} & $\bm{0.74\pm0.21}$ & ---     \\ \cmidrule(l){2-12}
\multirow{-2}{*}{CIFAR-10}  & Perf. (\%) & $91.3\pm0.2$   & $91.4\pm0.3$   & $90.2\pm0.9$   & $89.2\pm3.3$   & $73.5\pm7.2$   & $92.2\pm0.4$   & $91.8\pm0.2$   & $75.2\pm4.3$                                                  & $\bm{92.9\pm0.7}$  & $93.5$    \\ \midrule
                            & Corr. & $-0.05\pm0.03$ & $-0.07\pm0.05$ & $0.08\pm0.11$  & $-0.13\pm0.02$ & $-0.28\pm0.05$ & $-0.15\pm0.07$ & $-0.08\pm0.04$ & \cellcolor[HTML]{FFFFFF}{\color[HTML]{333333} $-0.09\pm0.03$} & $\bm{0.63\pm0.13}$ & ---     \\ \cmidrule(l){2-12}
\multirow{-2}{*}{CIFAR-100} & Perf. (\%) & $71.1\pm1.4$   & $69.5\pm2.8$   & $67.9\pm1.8$   & $64.9\pm2.2$   & $59.0\pm4.1$   & $70.1\pm0.6$   & $68.8\pm0.6$   & $51.3\pm6.1$                                                  & $\bm{72.4\pm1.7}$  & $72.9$    \\ \bottomrule
\end{tabular}
}}}
\caption{\label{tab:images}The Spearman’s rank correlation of architecture's performance (Corr.) and the test performance of the best architecture selected on the condensed dataset (Perf.) on two \textbf{image datasets}. Grid search is applied to find the best architecture.}
\end{table*}
\begin{table*}[!htbp]
\adjustbox{max width=0.98\linewidth}{%
{\renewcommand{\arraystretch}{0.8}%
\begin{tabular}{ccccccccccc} \toprule
    \multirow{2}{*}{Dataset}    & Ratio             & \multicolumn{2}{c}{Random}  & \multicolumn{2}{c}{GCond-X}      & \multicolumn{2}{c}{GCond}        & \multicolumn{2}{c}{\ouralgo}                   & Whole Graph                   \\ 
                                 & ($c_{\mathrm{train}}/n$)   & Corr.         & Perf. (\%)   & Corr.         & Perf. (\%)        & Corr.         & Perf. (\%)        & Corr.                  & Perf. (\%)             & Perf. (\%)                    \\ \midrule
    \multirow{3}{*}{Cora}       & 0.9\%             & $0.29\pm.08$  & $81.2\pm1.1$ & $0.16\pm.07$  & $79.5\pm0.7$      & $0.61\pm.03$  & $81.9\pm1.6$      & $\mathbf{0.80\pm.03}$  & $\mathbf{83.0\pm0.2}$  & \multirow{3}{*}{$83.8\pm0.4$} \\
                                & 1.8\%             & $0.40\pm.04$  & $81.9\pm0.5$ & $0.21\pm.07$  & $80.3\pm0.4$      & $0.76\pm.06$  & $83.2\pm0.9$      & $\mathbf{0.85\pm.03}$  & $\mathbf{83.4\pm0.2}$  &                               \\
                                & 3.6\%             & $0.51\pm.04$  & $82.2\pm0.6$ & $0.23\pm.04$  & $80.9\pm0.6$      & $0.81\pm.04$  & $83.2\pm1.1$      & $\mathbf{0.90\pm.01}$  & $\mathbf{83.4\pm0.3}$  &                               \\ \midrule
    \multirow{3}{*}{Citeseer}   & 1.3\%             & $0.38\pm.11$  & $71.9\pm0.8$ & $0.15\pm.07$  & $70.7\pm0.9$      & $0.68\pm.03$  & $71.3\pm1.2$      & $\mathbf{0.79\pm.01}$  & $\mathbf{73.1\pm0.2}$  & \multirow{3}{*}{$73.7\pm0.6$} \\
                                & 2.6\%             & $0.56\pm.06$  & $72.2\pm0.4$ & $0.29\pm.05$  & $70.8\pm0.5$      & $0.79\pm.05$  & $71.5\pm0.7$      & $\mathbf{0.83\pm.02}$  & $\mathbf{73.3\pm0.5}$  &                               \\
                                & 5.2\%             & $0.71\pm.05$  & $73.0\pm0.3$ & $0.35\pm.08$  & $70.2\pm0.4$      & $0.83\pm.03$  & $71.1\pm0.8$      & $\mathbf{0.89\pm.02}$  & $\mathbf{73.4\pm0.4}$  &                               \\ \midrule
    \multirow{3}{*}{Ogbn-arxiv} & 0.1\%             & $0.59\pm.08$  & $70.1\pm1.7$ & $0.39\pm.06$  & $69.8\pm1.4$      & $0.59\pm.07$  & $70.3\pm1.4$      & $\mathbf{0.77\pm.04}$  & $\mathbf{71.9\pm0.8}$  & \multirow{3}{*}{$73.2\pm0.8$} \\
                                & 0.25\%            & $0.63\pm.05$  & $70.3\pm1.3$ & $0.44\pm.03$  & $70.1\pm0.7$      & $0.64\pm.05$  & $70.5\pm1.0$      & $\mathbf{0.83\pm.03}$  & $\mathbf{72.4\pm1.0}$  &                               \\
                                & 0.5\%             & $0.68\pm.07$  & $70.9\pm1.0$ & $0.47\pm.05$  & $70.0\pm0.7$      & $0.67\pm.05$  & $71.1\pm0.6$      & $\mathbf{0.88\pm.03}$  & $\mathbf{72.6\pm0.6}$  &                               \\ \midrule
    \multirow{3}{*}{Reddit}     & 0.1\%             & $0.42\pm.09$  & $92.1\pm1.6$ & $0.39\pm.04$  & $90.9\pm0.8$      & $0.53\pm.06$  & $90.9\pm1.7$      & $\mathbf{0.79\pm.03}$  & $\mathbf{92.1\pm0.9}$  & \multirow{3}{*}{$94.1\pm0.7$} \\
                                & 0.25\%            & $0.50\pm.06$  & $92.7\pm1.3$ & $0.41\pm.05$  & $90.9\pm0.5$      & $0.61\pm.04$  & $91.2\pm1.2$      & $\mathbf{0.83\pm.01}$  & $\mathbf{92.9\pm0.7}$  &                               \\
                                & 0.5\%             & $0.58\pm.06$  & $92.8\pm0.7$ & $0.42\pm.03$  & $91.5\pm0.6$      & $0.66\pm.02$  & $92.1\pm0.9$      & $\mathbf{0.87\pm.01}$  & $\mathbf{93.1\pm0.5}$  &                               \\ \bottomrule               
\end{tabular}
}}
\caption{\label{tab:graphs}Spearman's rank correlation of convolution filters in GNNs (Corr.) and the test performance of the best convolution filter selected on the condensed graph (Pref.) on four \textbf{graph datasets}. Continuous hyperparameter optimization~\citep{lorraine2020optimizing} is applied to find the best convolution filter, while Spearman's rank correlation coefficients are evaluated on $80$ sampled hyperparameter configurations. $n$ is the total number of nodes in the original graph, and $c_{\mathrm{train}}$ is the number of training nodes in the condensed graph.}
\end{table*}
\begin{table*}[t]
\adjustbox{max width=\textwidth}{%
{\renewcommand{\arraystretch}{1.2}%
{\scriptsize
\begin{tabular}{ccccccccc}
\toprule
\multirow{2}{*}{NAS Algorithm} & \multicolumn{2}{c}{Random} & \multicolumn{2}{c}{DC}     & \multicolumn{2}{c}{HCDC}   & \multicolumn{2}{c}{Original} \\ 
                        & Time (sec)   & Perf. (\%)       & Time (sec)   & Perf. (\%)        & Time (sec)   & Perf. (\%)       & Time (sec)    & Perf. (\%)         \\ \midrule
DARTS-PT                & $37.1$ & $89.4\pm0.3$ & $39.2$ & $85.2\pm1.9$ & $35.5$ & $91.9\pm0.4$ & $229$   & $92.7\pm0.6$  \\
REINFORCE               & $166$   & $88.1\pm1.8$ & $105$   & $80.1\pm6.5$ & $119$   & $92.3\pm1.1$ & $1492$  & $93.0\pm0.9$ \\ \bottomrule
\end{tabular}
}}}
\caption{\label{tab:nas}The search time and test performance of the best architecture find by NAS methods on the condensed datasets. We consider two NAS algorithms: (1) the differentiable NAS algorithm DARTS-PT and (2) REINFORCE without parameter-sharing.}
\end{table*}

\section{Experiments}
\label{sec:experiments}
In this section, we validate the effectiveness of hyperparameter calibrated dataset condensation (\ouralgo) when applied to speed up architecture/hyperparameter search on two types of data: \textbf{images} and \textbf{graphs}.
For an ordered list of architectures, we calculate Spearman's rank correlation coefficient $-1\leq \text{Corr.} \leq1$, between the rankings of their validation performance on the original and condensed datasets.
This correlation coefficient (denoted by \textbf{Corr.}) indicates how similar the performance ranking on the condensed dataset is to that on the original dataset.
We also report the test accuracy (referred to as \textbf{Perf.}) evaluated on the original dataset of the architectures selected on the condensed dataset.
If the test performance is close to the true optimal performance among all architectures, we say the architecture search outcome is preserved with high accuracy.
See~\cref{apd:implementations} and~\cref{apd:implement-detials} for more discussions on implementation and experimental setups.

\begin{figure}[!htbp]
    \includegraphics[width=0.65\columnwidth]
    {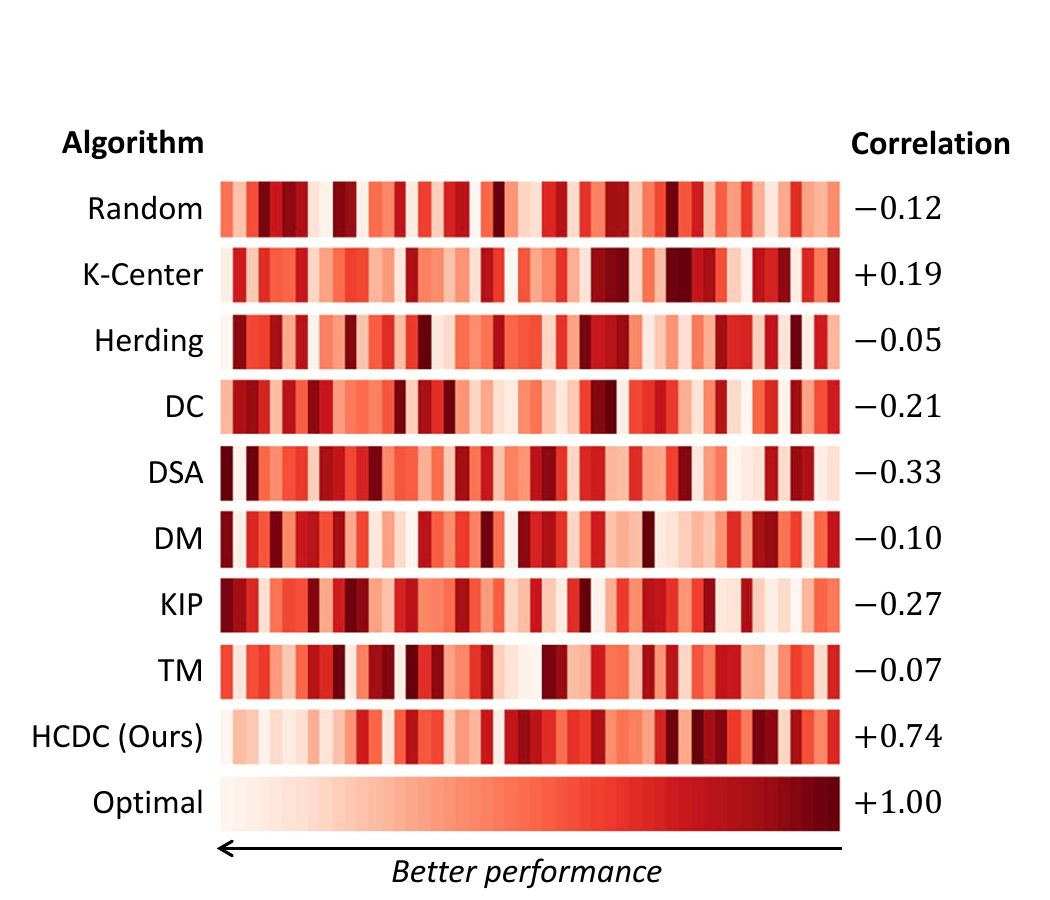}
    \caption{\label{fig:ranks}Visualization of the performance rankings of architectures (subsampled from the search space) evaluated on different condensed datasets. Colors indicate the performance ranking on the original dataset, while lighter shades refer to better performance. Spearman's rank correlations are shown on the right.}
\end{figure}

\para{Preserving architecture performance ranking on images.}
We follow the practice of~\citep{cui2022dc} and construct the search space by sampling 100 networks from NAS-Bench-201~\citep{dong2020bench}, which contains the ground-truth performance of 15,625 networks.
All models are trained on CIFAR-10 or CIFAR-100 for 50 epochs under five random seeds and ranked according to their average accuracy on a held-out validation set of $10$K images.
As a common practice in NAS~\citep{liu2018darts}, we reduce the number of repeated blocks in all architecture from 15 to 3 during the search phase, as deep models are hard to train on the small condensed datasets.
We consider three coreset baselines, including uniform random sampling, K-Center~\citep{farahani2009facility}, and Herding~\citep{welling2009herding} coresets, as well as five standard condensation baselines, including dataset condensation (DC)~\citep{zhao2020dataset}, differentiable siamese augmentation (DSA)~\citep{zhao2021dataseta}, distribution matching (DM)~\citep{zhao2021datasetb}, Kernel Inducing Point (KIP)~\citep{nguyen2020dataset,nguyen2021dataset}, and Training Trajectory Matching (TM)~\citep{cazenavette2022dataset}.
For the coreset and condensation baselines, we randomly split the condensed dataset to obtain the condensed validation data while keeping the train-validation split ratio.
We subsample or compress the original dataset to 50 images per class for all baselines.
As shown in~\cref{tab:images},  our~\ouralgo is much better at preserving the performance ranking of architectures compared to all other coreset and condensation methods.
At the same time, \ouralgo also consistently attains better-selected architectures' performance, which implies the \ouralgo condensed datasets are reliable proxies of the original datasets for architecture search.



\para{Speeding up architecture search on images.}
We then combine \ouralgo with some off-the-shelf NAS algorithms to demonstrate the efficiency gain when evaluated on the proxy condensed dataset.
We consider two NAS algorithms: DARTS-PT~\citep{wang2020rethinking}, which is a parameter-sharing based differentiable NAS algorithm, and REINFORCE~\citep{williams1992simple}, which is a reinforcement learning algorithm without parameter sharing.
In~\cref{tab:nas}, we see all coreset/condensation baselines can bring significant speed-ups to the NAS algorithms since the models are trained on the small proxy datasets.
Same as under the grid search setup, the test performance of the selected architecture on the~\ouralgo condensed dataset is consistently higher.
Here a small search space of 100 sampled architectures is used as in~\cref{tab:nas} and we expect even higher efficiency gain on larger search spaces.

\begin{figure}[!htbp]
 \includegraphics[width=0.95\columnwidth]{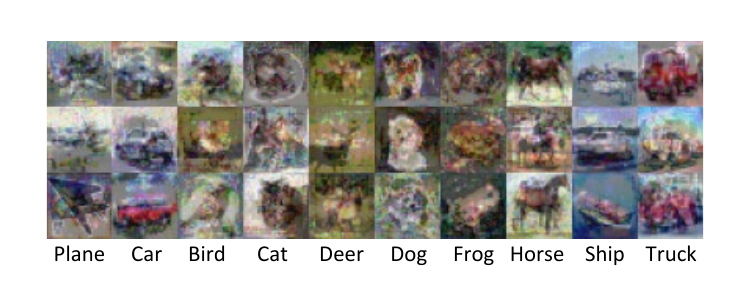}
  \caption{\label{fig:visualization}Visualization of some example condensed validation set images using our HCDC algorithm on CIFAR-10.}
\end{figure}

In~\cref{fig:ranks}, we directly visualize the performance rankings of architectures on different condensed datasets. Each color slice indicates one architecture and and are re-ordered with the ranking from the condensation algorithm. rows that are more similar to the 'optimal' gradient indicate that the algorithm is ranking the architectures similarly to the optimial ranking. The best architectures with the HCDC algorithm are among the best in the original dataset.

\para{Finding the best convolution filter on graphs.}
We now consider the search space of graph neural networks' (GNN) convolution filters, which is intrinsically continuous (i.e., defined by a few continuous hyperparameters which parameterize the convolution filter, see~\cref{sec:algorithm} for details).
Our goal is to speed up the selection of the best-suited convolution filter design on large graphs.
We consider 2-layer message-passing GNNs whose convolution matrix is a truncated sum of powers of the graph Laplacian; see~\cref{apd:implement-detials}.
Four node classification graph benchmarks are used, including two small graphs (Cora and Citeseer) and two large graphs (Ogbn-arxiv and Reddit) with more than $100$K nodes.
To compute the Spearman's rank correlations, we sample $80$ hyperparameter setups from the search space and compare their performance rankings.
We test \ouralgo against three baselines: (1) Random: random uniform sampling of nodes and find their induced subgraph, (2) GCond-X: graph condensation~\citep{jin2021graph} but fix the synthetic adjacency to the identity matrix, (3) GCond: graph condensation which also learns the adjacency. 
The whole graph performance is oracle and shows the best possible test performance when the convolution filter is optimized on the original datasets using hypergradient-based method~\citep{lorraine2020optimizing}.
In \cref{tab:graphs}, we see~\ouralgo consistently outperforms the other approaches, and the test performance of selected architecture is close to the ground-truth optimal.

\begin{figure}[!htbp]
    \includegraphics[width=0.55\columnwidth,trim={5pt 0pt 0pt 5pt},clip]{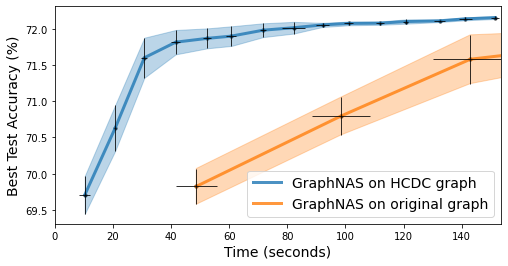}
    \caption{\label{fig:graph-nas}Speed-up of graph NAS's search process, when evaluated on the small proxy dataset condensed by \ouralgo.}
\end{figure}

\textbf{Speeding up off-the-shelf graph architecture search algorithms.}
Finally, we demonstrate \ouralgo can speed up off-the-shelf graph architecture search methods.
We use graph NAS~\citep{gao2019graphnas} on Ogbn-arxiv with a compression ratio of $c_{\mathrm{train}}/n=0.5\%$, where $n$ is the size of the original graph, and $c_{\mathrm{train}}$ is the number of training nodes in the condensed graph.
The search space of GNN architectures is the same as in~\citep{gao2019graphnas}, where various attention and aggregation functions are incorporated.
We plot the best test performance of searched architecture versus the search time in~\cref{fig:graph-nas}.
We see that when evaluated on the dataset condensed by \ouralgo, the search algorithm finds the better architectures much faster.
This efficiency gain provided by the small proxy dataset is orthogonal to the design of search strategies and should be applied to any type of data, including graph and images.

\section{Conclusion}
\label{sec:conclusion}

We propose a hyperparameter calibration formulation for dataset condensation to preserve the outcome of hyperparameter optimization, 
which is then solved by aligning the hyperparameter gradients.
We demonstrate both theoretically and experimentally that \ouralgo can effectively preserve the validation performance rankings of architectures and accelerate the hyperparameter/architecture search on images and graphs.
The overall performance of \ouralgo can be affected by
\begin{enumerate*}[label=(\arabic*)]
    \item how the differentiable NAS model used for condensation generalizes to unseen architectures,
    \item where we align hypergradients in the search space,
    \item how we learn the synthetic training set,
    \item how we parameterize the synthetic dataset,
\end{enumerate*}
and we leave the exploration of these design choices to future work.
We hope our work opens up a promising avenue for speeding up hyperparameter/architecture search by dataset compression.

\section*{Acknowledgement}
Ding, Xu, Rabbani, Liu, and Huang are supported by DARPA Transfer from Imprecise and Abstract Models to Autonomous Technologies (TIAMAT) 80321, National Science Foundation NSF-IIS-2147276 FAI, DOD-ONR-Office of Naval Research under award number N00014-22-1-2335, DOD-AFOSR-Air Force Office of Scientific Research under award number FA9550-23-1-0048, DOD-DARPA-Defense Advanced Research Projects Agency Guaranteeing AI Robustness against Deception (GARD) HR00112020007, Adobe, Capital One and JP Morgan faculty fellowships.

\clearpage
\newpage

\appendix

\section{A Detailed Diagram of the Proposed HCDC}\label{apd:diagram}

\begin{figure}[!htbp]
\vspace{-1em}
\centering
\includegraphics[width=\textwidth]{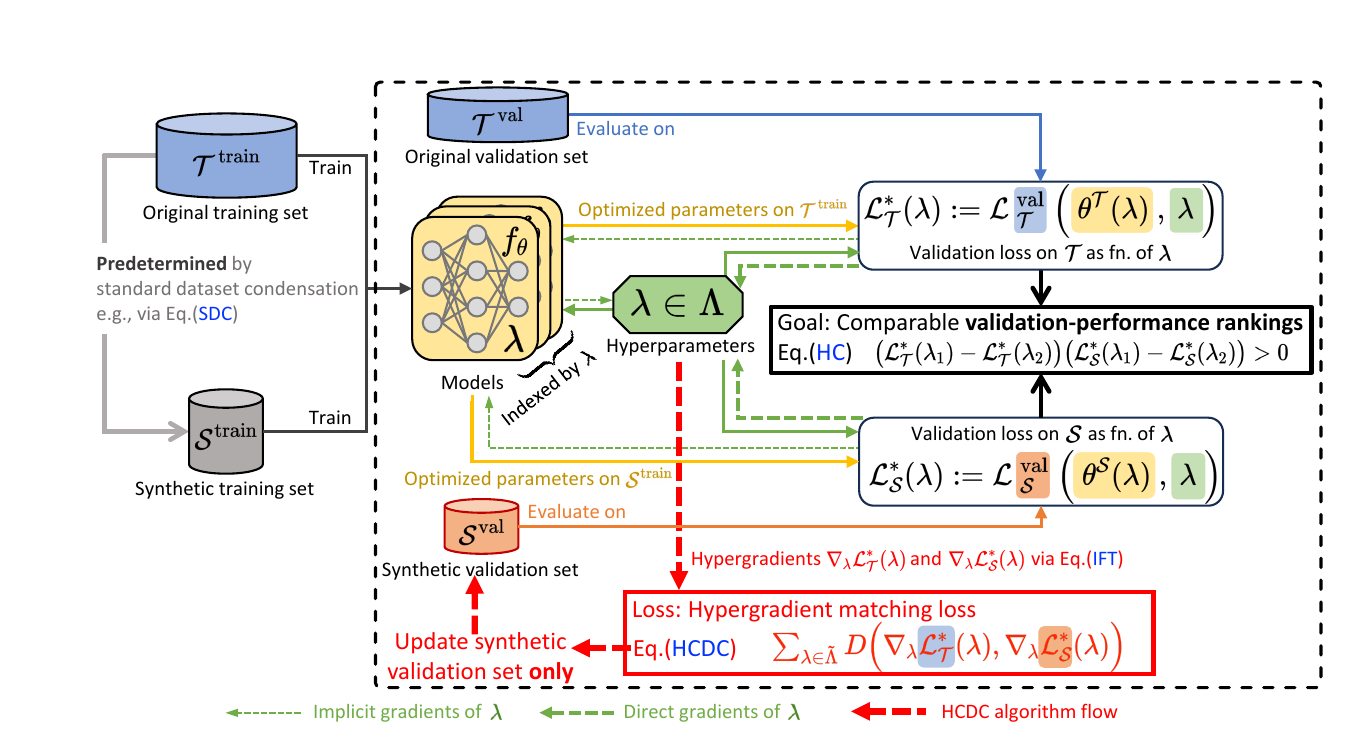}
\caption{Hyperparameter Calibrated Dataset Condensation (HCDC) aims to find a small \textcolor{orange}{validation} \textcolor{orange}{dataset} such that the validation-performance rankings of the models with different hyperparameters are comparable to the large original dataset's \big(\cref{eq:hc} in the ``{Goal}'' box\big). Our method realizes this goal by learning the synthetic validation set to match the hypergradients w.r.t the hyperparameters \big(\cref{eq:hcdc} in the \textcolor{red}{``{Loss}'' box}\big).   
\textbf{Our contribution is depicted within the big black dashed box.} The algorithm flow is indicated through the \textcolor{red}{red dashed arrows}. Solid arrows (blue, yellow and green) indicate forward passes.
To calculate the hypergradients with respect to hyperparameters $\lambda$, we backpropagate to compute both \textcolor{YellowGreen}{implicit/direct gradients} (\textcolor{YellowGreen}{thin/thick green dashed arrows}). 
The synthetic training set is predetermined by any standard dataset condensation (SDC) methods (e.g.,~\cref{eq:sdc-bl}).
The synthetic training and validation datasets obtained can later be used for hyperparameter search using only a fraction of the original computational load.
}
\label{fig:schematic-appendix}
\end{figure}

\section{More Related Work}
\label{apd:extend-related}

This section contains extensive discussions of many related areas, which cannot be fitted into the main paper due to the page limit.

\subsection{Dataset Condensation and Coreset Selection}

Firstly, we review the two main approaches to reducing the training set size while preserving model performance.

\textbf{Dataset condensation} (or distillation) is first proposed in~\citep{wang2018dataset} as a learning-to-learn problem by formulating the network parameters as a function of synthetic data and learning them through the network parameters to minimize the training loss over the original data. However, the nested-loop optimization precludes it from scaling up to large-scale in-the-wild datasets. \citet{zhao2020dataset} alleviate this issue by enforcing the gradients of the synthetic samples w.r.t. the network weights to approach those of the original data, which successfully alleviates the expensive unrolling of the computational graph. Based on the meta-learning formulation in~\citep{wang2018dataset}, \citet{bohdal2020flexible} and \citet{nguyen2020dataset,nguyen2021dataset} propose to simplify the inner-loop optimization of a classification model by training with ridge regression which has a closed-form solution, while \citet{such2020generative} model the synthetic data using a generative network. To improve the data efficiency of synthetic samples in gradient-matching algorithm, \citet{zhao2021dataseta} apply differentiable Siamese augmentation, and~\citet{kim2022dataset} introduce efficient synthetic-data parametrization. Recently, a new distribution-matching framework~\citep{zhao2021datasetb} proposes to match the hidden features rather than the gradients for fast optimization but may suffer from performance degradation compared to gradient-matching~\citep{zhao2021datasetb}, where~\citet{kim2022dataset} provide some interpretation.

\textbf{Coreset selection} methods choose samples that are important for training based on heuristic criteria, for example, minimizing the distance between coreset and whole-dataset centers~\citep{chen2010super,rebuffi2017icarl}, maximizing the diversity of selected samples in the gradient space~\citep{aljundi2019gradient}, discovering cluster centers~\citep{sener2018active}, and choosing samples with the largest negative implicit gradient~\citep{borsos2020coresets}. \textit{Forgetting}~\citep{toneva2018empirical} measures the forgetfulness of trained samples and drops those that are not easy to forget. \textit{GraNd}~\citep{paul2021deep} selects the training samples that contribute most to the training loss in the first few epochs. \textit{Prism}~\citep{kothawade2022prism} select samples to maximize submodular set-functions, which are combinatorial generalizations of entropy measures~\citep{iyer2021submodular}. Recent benchmark~\citep{guo2022deepcore} of a variety of coreset selection methods for image classification indicates \textit{Forgetting}, \textit{GraNd}, and \textit{Prism} are among the best-performing corset methods but still evidently underperform the dataset condensation baselines. Although coreset selection can be very efficient, most of the methods above suffer from three major limitations:
\begin{enumerate*}[label=(\arabic*)]
    \item their performance is upper-bounded by the information in the selected samples;
    \item most of them do not directly optimize the synthetic samples to preserve the model performance; and 
    \item most of them select samples incrementally and greedily, which is short-sighted.
\end{enumerate*}

\subsection{Implicit Differentiation and Differentiable NAS}

Secondly, we list two relevant areas to this work: implicit differentiation methods based on the implicit function theorem (IFT), and the differentiable neural architecture search (NAS) algorithms.

\textbf{Implicit differentiation} methods apply the implicit function theorem (IFT) to the nested-optimization problems~\citep{ochs2015bilevel,wang2019solving}. The IFT requires inverting the training Hessian with respect to the network weights, where early work either computes the inverse explicitly~\citep{bengio2000gradient,larsen1996design} or approximates it as the identity matrix~\citep{luketina2016scalable}. Conjugate gradient (CG) is applied to invert the Hessian approximately~\citep{pedregosa2016hyperparameter}, but is difficult to scale to deep networks. Several methods have been proposed to efficiently approximate Hessian inverse, for example, 1-step unrolled differentiation~\citep{luketina2016scalable}, Fisher information matrix~\citep{larsen1996design}, NN-structure aided Kronecker-factored inversion~\citep{martens2015optimizing}. \citet{lorraine2020optimizing} use the Neumann inverse approximation, which is a stable alternative to CG~\citep{shaban2019truncated} and successfully scale gradient-based bilevel-optimization to large networks with constant memory constraint. It is shown that unrolling differentiation around locally optimal weights for $i$ steps is equivalent to approximating the Neumann series inverse approximation up to the first $i$ terms.

\textbf{Differentiable NAS} methods, e.g., DARTS~\citep{liu2018darts} explores the possibility of transforming the discrete neural architecture space into a continuously differentiable form and further uses gradient optimization to search the neural architecture. DARTS follows a cell-based search space~\citep{zoph2018learning} and continuously relaxes the original discrete search strategy. Despite its simplicity, several works cast double on the effectiveness of DARTS~\citep{li2020random,zela2019understanding}. SNAS~\citep{xie2018snas} points out that DARTS suffers from the unbounded bias issue towards its objective, and it remodels the NAS and leverages the Gumbel-softmax trick~\citep{jang2016categorical,maddison2016concrete} to learn the exact architecture parameter. Differentiable NAS techniques have also been applied to graphs to design data-specific GNN architectures~\citep{wang2021autogel,huan2021search} automatically.

\subsection{Graph Reduction}

Lastly, when we apply \ouralgo to graph data, it relates to the graph reduction method for graph neural network training, which we summarize as follows.

\textbf{Graph coreset selection} is a non-trivial generalization of the above method coreset methods given the non-\textit{iid} nature of graph nodes and the non-linearity nature of GNNs. The very few off-the-shelf graph coreset algorithms are designed for graph clustering~\citep{baker2020coresets,braverman2021coresets} and are not optimal for the training of GNNs.

\textbf{Graph condensation}~\citep{jin2021graph} achieves the state-of-the-art performance for preserving GNNs' performance on the simplified graph. However, \citet{jin2021graph}~only adapt the gradient-matching algorithm of dataset condensation~\citet{zhao2020dataset} to graph data, together with an MLP-based generative model for edges~\citep{anand2018generative,simonovsky2018graphvae}, leaving out several major issues on efficiency, performance, and generalizability. Subsequent work aims to apply the more efficient distribution-matching algorithm~\citep{zhao2021datasetb,wang2022cafe} of dataset condensation to graph~\citep{liu2022graph} or speed up gradient-matching graph condensation by reducing the number of gradient-matching-steps~\citep{jin2022condensing}. While the efficiency issue of graph condensation is mitigated~\citep{jin2022condensing}, the performance degradation on medium- and large-sized graphs still renders graph condensation practically meaningless. Our \ouralgo is specifically designed for repeated training in architecture search, which is, in contrast, well-motivated.

\textbf{Graph sampling} methods~\citep{chiang2019cluster,zeng2019graphsaint} can be as simple as uniformly sampling a set of nodes and finding their induced subgraph, which is understood as a graph-counterpart of uniform sampling of \textit{iid} samples. However, most of the present graph sampling algorithms (e.g., ClusterGCN~\citep{chiang2019cluster} and GraphSAINT~\citep{zeng2019graphsaint}) are designed for sampling multiple subgraphs (mini-batches), which form a cover of the original graph for training GNNs with memory constraint. Therefore those graph mini-batch sampling algorithms are effectively graph partitioning algorithms and not optimized to find just one representative subgraph. 

\textbf{Graph sparsification}~\citep{batson2013spectral,satuluri2011local} and \textbf{graph coarsening}~\citep{loukas2018spectrally,loukas2019graph,huang2021scaling,cai2020graph} algorithms are usually designed to preserve specific graph properties like graph spectrum and graph clustering. Such objectives are often not aligned with the optimization of downstream GNNs and are shown to be sub-optimal in preserving the information to train GNNs well~\citep{jin2021graph}.

\section{Preliminaries on Graph Hyperparameter Search}
\label{apd:prelim}

\ouralgo applies to not only image data but also graph data and graph neural networks (GNN). In~\cref{sec:experiments}, we use \ouralgo to search for the best convolution filter in message-passing GNNs. In this section, we give the backgrounds of the graph-related experiments. 

\subsection{Node Classification and GNNs}
Node classification on a graph considers a graph $\mathcal{T}=(A,X,\mathbf{y})$ with adjacency matrix $A\in\{0,1\}^{n\times n}$, node features $X\in\mathbb{R}^{n\times d}$, node class labels $\mathbf{y}$, and mutually disjoint node-splits $V_{\mathrm{train}}\bigcup V_{\mathrm{val}}\bigcup V_{\mathrm{test}}=[n]$.
Using a \emph{graph neural network} (GNN) $f_{\theta,\lambda}: \mathbb{R}_{\geq0}^{n\times n}\times \mathbb{R}^{n\times d}\to \mathbb{R}^{n\times K}$, where $\theta\in\Theta$ denotes the parameters and $\lambda\in\Lambda$ denotes the hyper-parameters (if they exist), we aim to find $\theta^\mathcal{T}=\arg\min_{\theta}\mathcal{L}^{\mathrm{{train}}}_{\mathcal{T}}(\theta,\lambda)$, where $\mathcal{L}^{\mathrm{train}}_{\mathcal{T}}(\theta,\lambda)\coloneqq\sum_{i\in V_{\mathrm{train}}}\ell([f_{\theta,\lambda}(A,X)]_i,y_i)$ and $\ell(\hat{y},y)$ is the cross-entropy loss.
The node classification loss $\mathcal{L}^{\mathrm{train}}_{\mathcal{T}}(\theta,\lambda)$ is under the \emph{transductive} setting, which can be easily generalized to the \emph{inductive} setting by assuming only $\{A_{ij}\mid i,j\in V_{\mathrm{train}}\}$ and $\{X_i\mid i\in V_{\mathrm{train}}\}$ are used during training.

\subsection{Additional Downstream Tasks}
\label{apd:prelim:task}

We have defined the downstream task this paper mainly focuses on, node classification on graphs. Where we are given a graph $\mathcal{T}=(A,X,\mathbf{y})$ with adjacency matrix $A\in\{0,1\}^{n\times n}$, node features $X\in\mathbb{R}^{n\times d}$, node class labels $\mathbf{y}\in[K]^n$, and mutually disjoint node-splits $V_{\text{train}}\bigcup V_{\text{val}}\bigcup V_{\text{test}}=[n]$, and the goal is to predict the node labels.

Here, we show that the settings above can also be used to describe per-pixel classification on images (e.g., for semantic segmentation) where CNNs are usually used. For \emph{per-pixel classification}, we are given a set of $\mathfrak{n}$ images of size $w\times h$, so the pixel values of the $j$-th image can be formatted as a tensor $\mathfrak{X}_j\in\mathbb{R}^{w\times h\times c}$ if there are $c$ channels. We are also given the pixel labels $\mathfrak{Y}_j\in[K]^{w\times h}$ for each image $j\in[\mathfrak{n}]$ and the mutually disjoint image-splits $I_{train}\bigcup I_{val}\bigcup I_{test}=[\mathfrak{n}]$. Clearly, we can reshape the pixel values and pixel labels of the $j$-th image to $wh\times c$ and $wh$, respectively, and concatenate those matrices from all images. Following this, denoting $n=\mathfrak{n}wh$, we obtain the concatenated pixel value matrix $X\in\mathbb{R}^{n\times c}$ and the concatenated pixel label vector $\mathbf{y}\in[K]^{n}$. The image-splits are translated into pixel-level splits where $V_{train}=\{i\mid (j-1)wh\leq i\le jwh, j\in I_{trian}\}$ (similar for $V_{val}$ and $V_{test}$) and $V_{train}\bigcup V_{val}\bigcup V_{test}=[n]$. We can also define the auxiliary adjacency matrix $A\in\{0,1\}^{n\times n}$ on the $n=\mathfrak{n}wh$ pixels, where $A$ is block diagonal $A=\diag{A_1,\ldots, A_{\mathfrak{n}}}$ and $A_j\in\{0,1\}^{wh\times wh}$ is the assumed adjacency (e.g., a two-dimensional grid) of the $j$-th image.

\subsection{Graph Neural Network Models}
\label{apd:prelim:nn}

Here we mainly focus on graph neural networks (GNNs) $f_{\theta,\lambda}: \mathbb{R}_{\geq0}^{n\times n}\times \mathbb{R}^{n\times d}\to \mathbb{R}^{n\times K}$, where $\theta\in\Theta$ denotes the parameters, and $\lambda\in\Lambda$ denotes the hyperparameters. In~\cref{sec:preliminaries}, we have seen that most GNNs can be interpreted as iterative convolution/message passing over nodes~\citep{ding2021vq,balcilar2021analyzing} where $X^{(0)}=X$ and $f(A,X)=X^{(L)}$, and for $l\in[L]$, the update-rule is,
\begin{equation}
\label{eq:gnn-single}
    X^{(l+1)}=\sigma\Big(C_{\alpha^{(l)}}(A)X^{(l)}W^{(l)}\Big),
\end{equation}
where $C_{\alpha^{(l)}}(A)$ is the convolution matrix parametrized by $\alpha^{(l)}$, $W^{(l)}$ is the learnable linear weights, and $\sigma(\cdot)$ denotes the non-linearity. Thus the parameters $\theta$ consists of all $\alpha$'s (if they exist) and $W$'s, i.e., $\theta=[\alpha^{(0)},\ldots,\alpha^{(L-1)}, W^{(0)},\ldots, W^{(L-1)}]$.

More specifically, it is possible for GNNs to have more than one convolution filter per layer~\citep{ding2021vq,balcilar2021analyzing} and we may generalize~\cref{eq:gnn-single} to,
\begin{equation}
\label{eq:gnn-multiple}
    X^{(l+1)}=\sigma\left(\sum_{i=1}^{p}C^{(i)}_{\alpha^{(l,i)}}(A)X^{(l)}W^{(l,i)}\right).
\end{equation}
Within this common framework, GNNs differ from each other by choice of convolution filters $\{C^{(i)}\}$, which can be either fixed or learnable. If $C^{(i)}$ is fixed, there is no parameters $\alpha^{(l,i)}$ for any $l\in[L]$. If $C^{(i)}$ is learnable, the convolution matrix relies on the learnable parameters $\alpha^{(l, i)}$ and can be different in each layer (thus should be denoted as $C^{(l, i)}$). Usually, for GNNs, the convolution matrix depends on the parameters in two possible ways: (1) the convolution matrix $C^{(l, i)}$ is scaled by the scalar parameter $\alpha^{(l, i)}\in\mathbb{R}$, i.e., $C^{(l,i)}=\alpha^{(l, i)}\mathfrak{C}^{(i)}$ (e.g., GIN~\citep{xu2018powerful}, ChebNet~\citep{defferrard2016convolutional}, and SIGN~\citep{frasca2020sign}); or (2) the convolution matrix is constructed by node-level self-attentions $[C^{(l,i)}]_{ij}=h_{\alpha^{(l, i)}}\big(X^{(l)}_{i,:}, X^{(l)}_{j,:}\big)[\mathfrak{C}^{(i)}]_{i,j}$ (e.g., GAT~\citep{velivckovic2017graph}, Graph Transformers~\citep{rong2020self,puny2020graph,zhang2020graph}). Based on~\citep{ding2021vq,balcilar2021analyzing}, we summarize the popular GNNs reformulated into the convolution over nodes / message-passing formula (\cref{eq:gnn-multiple}) in~\cref{tab:gnns}.

Convolutional neural networks can also be reformulated into the form of~\cref{eq:gnn-multiple}. For simplicity, we only consider a one-dimensional convolution neural network (1D-CNN), and the generalization to 2D/3D-CNNs is trivial. If we denote the constant cyclic permutation matrix (which corresponds to a unit shift) as $P\in\mathbb{R}^{n\times n}$, the update rule of a 1D-CNN with kernel size $(2K+1), K\geq 0$ can be written as,
\begin{equation}
\label{eq:cnn}
    X^{(l+1)}=\sigma\Big(\sum_{k=-K}^{k=K}\alpha_{k} P^kX^{(l)}W^{(l,k)}\Big).
\end{equation}
We will use this common convolution formula of GNNs (\cref{eq:gnn-multiple}) and 1D-CNNs (\cref{eq:cnn}) in \cref{apd:prelim:linear} and~\cref{thm:sdc-1dcnns}.

\begin{table}[t]
\centering
\caption{\label{tab:gnns}Summary of GNNs formulated as generalized graph convolution.}
\adjustbox{max width=\textwidth}{%
\begin{threeparttable}
\renewcommand*{\arraystretch}{1.5}
\begin{tabular}{Sc Sc Sc Sc Sl} \toprule
Model Name & Design Idea & Conv. Matrix Type & \# of Conv. & Convolution Matrix \\ \cmidrule{1-5}\morecmidrules\cmidrule{1-5}
GCN\tnote{1}~~\citep{kipf2016semi} & Spatial Conv. & Fixed & $1$ & $C=\wt{D}^{-1/2}\wt{A}\wt{D}^{-1/2}$ \\
SAGE-Mean\tnote{2}~~\citep{hamilton2017inductive} & Message Passing & Fixed & $2$ & \multicolumn{1}{l}{$\left\{\begin{tabular}[c]{@{}l@{}} $C^{(1)}=I_n$\\ $C^{(2)}=D^{-1}A$ \end{tabular}\right.\kern-\nulldelimiterspace$} \\
GAT\tnote{3}~~\citep{velivckovic2017graph} & Self-Attention & Learnable & \# of heads & \multicolumn{1}{l}{$\left\{\begin{tabular}[c]{@{}l@{}} $\fC^{(s)}=A+I_n$\itand\\ $h^{(s)}_{\va^{(l,s)}}(X^{(l)}_{i,:}, X^{(l)}_{j,:})=\exp\big(\mathrm{LeakyReLU}($\\ \quad$(X^{(l)}_{i,:}W^{(l,s)}\concat X^{(l)}_{j,:}W^{(l,s)})\cdot\va^{(l,s)})\big)$ \end{tabular}\right.\kern-\nulldelimiterspace$} \\
GIN\tnote{1}~~\citep{xu2018powerful} & WL-Test & \multicolumn{1}{c}{\begin{tabular}[c]{@{}c@{}} Fixed $+$ \\ Learnable \end{tabular}} & $2$ & \multicolumn{1}{l}{$\left\{\begin{tabular}[c]{@{}l@{}} $C^{(1)}=A$\\ $\fC^{(2)}=I_n$\itand$h^{(2)}_{\epsilon^{(l)}}=1+\epsilon^{(l)}$ \end{tabular}\right.\kern-\nulldelimiterspace$} \\ \midrule
SGC\tnote{2}~~\citep{defferrard2016convolutional} & Spectral Conv. & Learnable & order of poly. & \multicolumn{1}{l}{$\left\{\begin{tabular}[c]{@{}l@{}} $\fC^{(1)}=I_n$, $\fC^{(2)}=2L/\lambda_{\max}-I_n$,~\\ $\fC^{(s)}=2\fC^{(2)}\fC^{(s-1)}-\fC^{(s-2)}$\\\itand$h^{(s)}_{\theta^{(s)}}=\theta^{(s)}$ \end{tabular}\right.\kern-\nulldelimiterspace$} \\
ChebNet\tnote{2}~~\citep{defferrard2016convolutional} & Spectral Conv. & Learnable & order of poly. & \multicolumn{1}{l}{$\left\{\begin{tabular}[c]{@{}l@{}} $\fC^{(1)}=I_n$, $\fC^{(2)}=2L/\lambda_{\max}-I_n$,~\\ $\fC^{(s)}=2\fC^{(2)}\fC^{(s-1)}-\fC^{(s-2)}$\\\itand$h^{(s)}_{\theta^{(s)}}=\theta^{(s)}$ \end{tabular}\right.\kern-\nulldelimiterspace$} \\
GDC\tnote{3}~~\citep{klicpera2019diffusion} & Diffusion & Fixed & $1$ & $C=S$ \\
Graph Transformers\tnote{4}~~\citep{rong2020self} & Self-Attention & Learnable & \# of heads & \multicolumn{1}{l}{$\left\{\begin{tabular}[c]{@{}l@{}} $\fC^{(s)}_{i,j}=1$\itand$h^{(s)}_{(W^{(l,s)}_Q, W^{(l,s)}_K)}(X^{(l)}_{i,:}, X^{(l)}_{j,:})$\\$=\exp\big(\frac{1}{\sqrt{d_{k,l}
}}(X^{(l)}_{i,:}W^{(l,s)}_Q)(X^{(l)}_{j,:}W^{(l,s)}_K)\tran\big)$  \end{tabular}\right.\kern-\nulldelimiterspace$} \\ \bottomrule
\end{tabular}
\appto\TPTnoteSettings{\large}
\begin{tablenotes}[para]
    \item[1] Where $\wt{A}=A+I_n$, $\wt{D}=D+I_n$.
    \item[2] $C^{(2)}$ represents mean aggregator. Weight matrix in~\citep{hamilton2017inductive} is $W^{(l)}=W^{(l,1)}\concat W^{(l,2)}$.
    \item[3] Need row-wise normalization. $C^{(l,s)}_{i,j}$ is non-zero if and only if $A_{i,j}=1$, thus GAT follows direct-neighbor aggregation.
    \item[4] The weight matrices of the two convolution supports are the same, $W^{(l,1)}=W^{(l,2)}$.
    \item[5] Where normalized Laplacian $L=I_n-D^{-1/2}AD^{-1/2}$ and $\lambda_{\max}$ is its largest eigenvalue, which can be approximated as $2$ for a large graph.
    \item[6] Where $S$ is the diffusion matrix $S=\sum_{k=0}^\infty\theta_k\bm{T}^k$, for example, decaying weights $\theta_k=e^{-t}\frac{t^k}{k!}$ and transition matrix $\bm{T}=\wt{D}^{-1/2}\wt{A}\wt{D}^{-1/2}$.
    \item[7] Need row-wise normalization. Only describes the global self-attention layer, where $W^{(l,s)}_Q, W^{(l,s)}_Q\in\sR^{f_l, d_{k,l}}$ are weight matrices which compute the queries and keys vectors. In contrast to GAT, all entries of $\fC^{(l,s)}_{i,j}$ are non-zero. Different design of Graph Transformers~\citep{puny2020graph,rong2020self,zhang2020graph} use graph adjacency information in different ways and is not characterized here; see the original papers for details.
\end{tablenotes}
\end{threeparttable}}
\end{table}

\subsection{\ouralgo is Applicable to Various Data, Tasks, and Models}
\label{apd:prelim:other}

In~\cref{apd:prelim:task,apd:prelim:nn}, we have discussed the formal definition of two possible tasks (1) node classification on graphs and (2) per-pixel classification on images, and reformulated many popular GNNs and CNNs into a general convolution form (\cref{eq:gnn-multiple,eq:cnn}). However, we want to note that the application of dataset condensation methods (including the standard dataset condensation~\citep{wang2018dataset,zhao2020dataset,zhao2021datasetb} and our \ouralgo) is not limited by the specific types of data, tasks, and models.

For \ouralgo, we can follow the conventions in~\citep{zhao2020dataset} to define the train/validation losses on \textit{iid} samples and define the notion of dataset condensation as learning a smaller synthetic dataset with less number of samples. Here we leave the readers to~\citep{zhao2020dataset} for formal definitions of condensation on datasets with \textit{iid} samples. 

More generally speaking, our \ouralgo can be applied as long as (1) the train and validation losses, i.e.,  $\mathcal{L}^{train}_{\mathcal{T}}(\theta, \lambda)$ and $\mathcal{L}^{val}_{\mathcal{T}}(\theta, \lambda)$ can be defined (as functions of the parameters and hyperparameters); and (2) we have a well-defined notion of the learnable synthetic dataset $\mathcal{S}$, (e.g., which includes prior-knowledge like what is the format of the synthetic data in $\mathcal{S}$ and how the same model $f_{\theta,\lambda}$ is applied).

\subsection{The Linear Convolution Regression Problem on Graph}
\label{apd:prelim:linear}

For the ease of theoretical analysis, in~\cref{thm:sdc-valid,thm:sdc-1dcnns,thm:sdc-overfit,thm:sdc-gnns} we consider a simplified \emph{linear convolution regression problem} as follows,
\begin{equation}
\label{eq:least-square-linear-conv}
    \theta^\mathcal{T}=\arg\min_{\theta=[\alpha, W]}\|C_{\alpha}(A)\ XW - \mathbf{y}\|^2
\end{equation}
where we are given continuous labels $\mathbf{y}$ and use sum-of-squares loss $\ell(\hat{y},y)=\|\hat{y}-y\|_2^2$ instead of the cross entropy loss used for node/pixel classification. We also assume a linear GNN/CNN $f_{\theta=[\alpha,W]}(A,X)=C_{\alpha}(A)XW$ is used, where $C_{\alpha}(A)$ is the \emph{convolution matrix} which depends on the adjacency matrix $A$ and the parameters $\alpha\in\mathbb{R}^{p}$, 
and $W$ is the learnable linear weights with $d$ elements (hence, the complete parameters consist of two parts, i.e., $\theta=[\alpha, W]$).

As explained in~\cref{apd:prelim:nn}, this \emph{linear convolution} model $f_{\theta=[\alpha,W]}(A,X)=C_{\alpha}(A)XW$ already generalizes a wide variety of GNNs and CNNs. 
For example, it can represent the (single-layer) graph convolution network (GCN)~\citep{kipf2016semi} whose convolution matrix is defined as
$C(A)=\tilde{D}^{-\frac{1}{2}}\tilde{A}\tilde{D}^{-\frac{1}{2}}$ where $\tilde{A}$ and $\tilde{D}$ are the ``self-loop-added'' adjacency and degree matrix (for GCNN there is no learnable parameters in $C_(A)$ so we omit $\alpha$). 
It also generalizes the one-dimensional \emph{convolution neural network} (1D-CNN), where the convolution matrix is $C_{\alpha}(A)=\sum_{k=-K}^{k=K}[\theta]_{k} P^k$ and $P$ is the cyclic permutation matrix correspond to a unit shift.

It is important to note that although we considered this simplified \emph{linear convolution regression problem} in some of our theoretical results, which is both convex and linear. We argue that most of the theoretical phenomena reflected by \cref{thm:sdc-valid,thm:sdc-1dcnns,thm:sdc-overfit,thm:sdc-gnns} can be generalized to the general non-convex losses and non-linear models; see~\cref{apd:proofs:non-convexlinear} for the corresponding discussions.

\section{Standard Dataset Condensation is Problematic Across GNNs}
\label{apd:problem}

In this section, we analyze that standard dataset condensation (SDC) is especially problematic when applied to graphs and GNNs, due to the poor generalizability across GNNs with different convolution filters.

For ease of theoretical discussions, in this subsection, we consider single-layer message-passing GNNs. Message passing GNNs can be interpreted as iterative convolution over nodes (i.e., \emph{message passing})~\citep{ding2021vq} where $X^{(0)}=X$, $X^{(l+1)}=\sigma(C_{\alpha^{(l)}}(A)X^{(l)}W^{(l)})$ for $l\in[L]$, and $f(A,X)=X^{(L)}$, where $C_{\alpha^{(l)}}(A)$ is the convolution matrix parametrized by $\alpha^{(l)}$, $W^{(l)}$ is the learnable linear weights, and $\sigma(\cdot)$ denotes the non-linearity.
One-dimensional convolution neural networks (1D-CNNs) can be expressed by a similar formula, $f(X)=(\sum_{k=-K}^{k=K}\alpha^{(k)} P^k)XW$, parameterized by $\theta=[\alpha,W]$ where $\alpha=[\alpha^{(-K)},\ldots,\alpha^{(K)}]$. $P$ is the cyclic permutation matrix (of a unit shift). The kernel size is $(2K+1), K\geq 0$; see~\cref{apd:prelim:nn} for details.

Despite the success of the gradient matching algorithm in preserving the model performance when trained on the condensed dataset~\citep{wang2022cafe}, it {naturally overfits} the model $f_{\theta,\lambda}$ used during condensation and generalizes poorly to others. 
There is no guarantee that the condensed synthetic data $\mathcal{S}^*$ which minimizes the objective (\cref{eq:sdc-gm}) for a specific model $f_{\theta,\lambda}$ (marked by its hyperparameter $\lambda$) can generalize well to other models $f_{\theta,\lambda'}$ where $\lambda'\neq\lambda$.
We aim to demonstrate that this overfitting issue can be much more \emph{severe on graphs} than on images, where our main theoretical results can be informally summarized as follows.

\begin{proposition-informal}
\label{thm:informal}
Standard dataset condensation using gradient matching algorithm (\cref{eq:sdc-gm}) is problematic across GNNs. The condensed graph using a single-layer message passing GNN may fail to generalize to the other GNNs with a different convolution matrix.
\end{proposition-informal}

We first show the successful generalization of SDC across one-dimensional \emph{convolution neural networks} (1D-CNN). 
Then, we show a contrary result on GNNs: failed generalization of SDC across GNNs. 
These theoretical analyses demonstrate the hardness of data condensation on graphs. 
Our analysis is based on the achievability condition of a gradient matching objective; see~\cref{assum:achievability} in~\cref{apd:problem}.

In \cref{thm:sdc-valid} of~\cref{apd:proofs:sdc-valid}, under least square regression with linear GNN/CNN (see~\cref{apd:prelim:linear} for formal definitions), if the standard dataset condensation GM objective is achievable, then the optimizer on the condensed dataset $\mathcal{S}$ is also optimal on the original dataset $\mathcal{T}$.
Now, we study the generalizability of the condensed dataset across different models. We first show a successful generalization of SDC across different 1D-CNN networks; see~\cref{thm:sdc-1dcnns} in~\cref{apd:problem}.
As long as we use a 1D-CNN with a sufficiently large kernel size $K$ during condensation, we can generalize the condensed dataset to a wide range of models, i.e., 1D-CNNs with a kernel size $K'\leq K$.

However, we obtain a contrary result for GNNs regarding the generalizability of condensed datasets across models.
Two dominant effects, which cause the failure of the condensed graph's ability to generalize across GNNs, are discovered.

Firstly, the learned adjacency $A'$ of the synthetic graph $\mathcal{S}$ can easily \emph{overfit} the condensation objective (see~\cref{thm:sdc-overfit}), and thus can fail to maintain the characteristics of the original structure and distinguish between different architectures; see~\cref{thm:sdc-overfit} in~\cref{apd:problem} for the theoretical result and~\cref{tab:adjacency-overfit} for relevant experiments.

\begin{figure*}[!htbp]
\begin{floatrow}
\capbtabbox[0.39\textwidth]{%
    \begin{subfloatrow}[1]
        \capbtabbox[0.37\textwidth][][t]{%
            \adjustbox{max width=0.30\textwidth}{%
            \centering\small
            \begin{tabular}{ccc} \toprule
            Ratio ($c/n$) & $A'$ learned  & $A'=I_c$                \\ \midrule
            0.05\%        & $59.2\pm 1.1$ & $\mathbf{61.3}\pm 0.5$   \\
            0.25\%        & $63.2\pm 0.3$ & $\mathbf{64.2}\pm 0.4$   \\ \bottomrule
            \end{tabular}
        }}{\caption{Test accuracy of graph condensation with learned or identity adjacency.}}
    \end{subfloatrow}
    \begin{subfloatrow}[1]
        \capbtabbox[0.37\textwidth][][t]{%
            \adjustbox{max width=0.35\textwidth}{%
            \centering\small
            \begin{tabular}{cccc} \toprule
            Condense\textbackslash Test & GCN          & SGC ($K=2$)  & GIN              \\ \midrule
            GCN  & $\textbf{60.3}\pm0.3$ & $59.2\pm0.7$ & $42.2\pm4.3$ \\
            SGC  & $59.2\pm1.1$ & $\textbf{60.5}\pm0.6$ & $39.0\pm7.1$ \\
            GIN  & $47.5\pm3.6$ & $43.6\pm5.8$ & $\textbf{59.1}\pm1.1$ \\ \bottomrule
            \end{tabular}
        }}{\caption{Generalization accuracy of graphs condensed with different GNNs (row) across GNNs (column) under $c/n=0.25\%$.}}
    \end{subfloatrow}
}{%
    \caption{\label{tab:adjacency-overfit}Test accuracy of GNNs trained on condensed Ogbn-arxiv~\citep{hu2020open} graph verifying the two effects (\cref{thm:sdc-overfit,thm:sdc-gnns}) that hinders the generalization of the condensed graph across GNNs. (a) Condensed adjacency is overfitted to the SDC Objective, (b) Convolution filters and inductive bias mismatch across GNNs.}
}
\ffigbox[0.58\textwidth]{%
    \begin{subfloatrow}[2]
        \ffigbox[0.28\textwidth]{%
            \includegraphics[width=\linewidth]{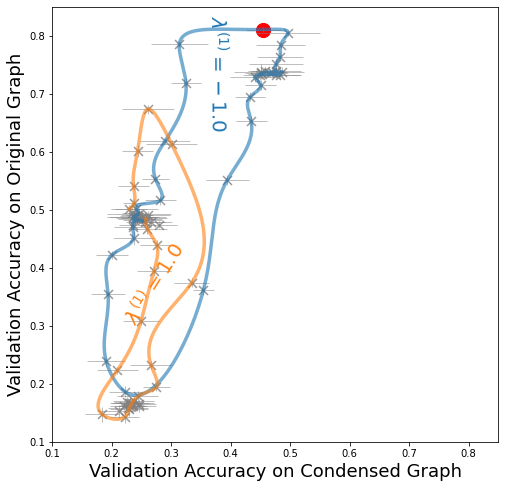}}{\caption{Condense ratio $c/n=0.2$}}
        \ffigbox[0.28\textwidth]{%
            \includegraphics[width=\linewidth]{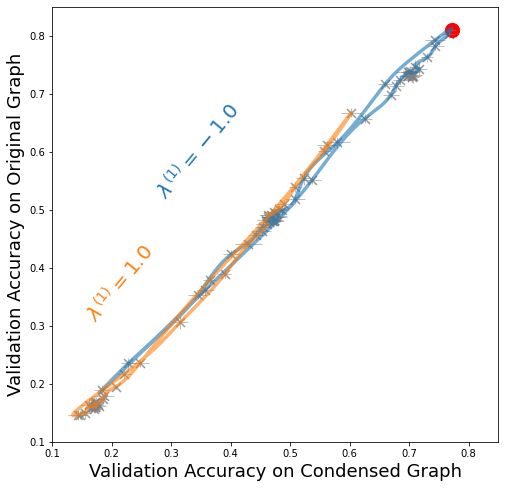}}{\caption{Condense ratio $c/n=0.8$}}
    \end{subfloatrow}
}{%
    \caption{\label{fig:gnn-model-space}The manifold of GNNs with convolution filters $C_{\lambda}\!=\!I\!+\!\lambda^{(1)} L\!+\!\lambda^{(2)}(\frac{2}{\lambda_{\max}}L-I)$ \big(linear combination of first two orders of ChebNet~\citep{defferrard2016convolutional}, $\lambda$'s are hyperparameters; see~\cref{apd:prelim:nn,apd:implement-detials}\big) projected to the plane of validation accuracy on condensed (x-axis) and original (y-axis) graphs under two ratios $c/n$ on Cora~\citep{yang2016revisiting}. The GNN with $C=(\frac{2}{\lambda_{\max}}\!-\!1)L\!\propto\! L$ (red dot) is a biased point in this model space.}
}
\end{floatrow}
\vspace{-8pt}
\end{figure*}

Secondly, GNNs differ from each other mostly on the design of convolution filter $C(A)$, i.e., how the convolution weights $C$ depend on the adjacency information $A$. The convolution filter $C(A)$ used during condensation is a single biased point in ``the space of convolutions''; see~\cref{fig:gnn-model-space} for a visualization, thus there is a \emph{mismatch} of inductive bias when transferring to a different GNN. These two effects lead to the obstacle when transferring the condensed graph across GNNs, which is formally characterized by~\cref{thm:sdc-gnns} in~\cref{apd:problem}.

\cref{thm:sdc-gnns} provides an effective lower bound on the relative estimation error of optimal model parameters when a different convolution filter $C'(\cdot)\neq C(\cdot)$ is used.\footnote{If $C'(\cdot)=C(\cdot)$ \cref{thm:sdc-valid} guarantees $W^{\mathcal{S}}_{C'}=W^{\mathcal{T}}_{C'}$ and the lower bound in~\cref{thm:sdc-gnns} is $0$.} According to the spectral characterization of convolution filters of GNNs (Table 1 of~\citep{balcilar2021analyzing}), we can approximately compute the maximum eigenvalue of $Q$ for some GNNs. For example, if we condense with $f^C$ graph isomorphism network (GIN-0)~\citep{xu2018powerful} but train $f^{C'}$ GCN on the condensed graph, we have $\|W^{\mathcal{S}}_{C'}-W^{\mathcal{T}}_{C'}\|/\|W^{\mathcal{T}}_{C'}\|\gtrapprox \overline{\mathrm{deg}}+1$ where $\overline{\mathrm{deg}}$ is the average node degree of the original graph. This large lower bound hints at the catastrophic failure when transferring across GIN and GCN; see~\cref{tab:adjacency-overfit}.

\begin{assumption}[Achievability of a gradient matching Objective]
\label{assum:achievability}
A gradient matching objective is defined to be achievable if there exists a non-degenerate trajectory $(\theta^{\mathcal{S}}_t)_{t=0}^{T-1}$ (i.e., a trajectory that spans the entire parameter space $\Theta$, i.e., $\mathrm{span}(\theta^{\mathcal{S}}_0,\ldots,\theta^{\mathcal{S}}_{T-1})\supseteq\Theta$), such that 
the gradient matching loss ({the objective of} \cref{eq:sdc-gm} {without expectation}) on this trajectory  is $0$.
\end{assumption}

\begin{proposition}[Successful Generalization of SDC across 1D-CNNs]
\label{thm:sdc-1dcnns}
Consider least-squares regression with one-dimensional linear convolution $f^{2K+1}(X)_{\theta}=(\sum_{k=-K}^{k=K}\alpha^{(k)} P^k)XW$ parameterized by $\theta=[\alpha,W]$ where $\alpha=[\alpha^{(-K)},\ldots,\alpha^{(K)}]$. $P$ is the cyclic permutation matrix (of a unit shift). The kernel size is $(2K+1), K\geq 0$. 
If the gradient matching objective of $f^{2K+1}$ is achievable, then the condensed dataset $\mathcal{S}^*$ achieves the gradient matching objective on any trajectory $\{\theta'^{\mathcal{S}}_t\}_{t=0}^{T-1}$ for any linear convolution $f^{2K'+1}_{\theta'}$ with kernel size $(2K'+1), K\geq K'\geq0$.
\end{proposition}

The intuition behind~\cref{thm:sdc-1dcnns} is that the 1D-CNN of kernel size $(2K+1)$ is a ``supernet'' of the 1D-CNN of kernel size $(2K'+1)$ if $K'\leq K$, and the condensed dataset via a bigger model can generalize well to smaller ones. 
This result suggests us to \emph{use a sufficiently large model during condensation}, to enable the generalization of the condensed dataset to a wider range of models.

\begin{proposition}[Condensed Adjacency Overfits SDC Objective]
\label{thm:sdc-overfit}
Consider least-squares regression with a linear GNN, 
$f(A, X)=C(A)XW$ parameterized by $W$ and $C(A)$ which depends on graph adjacency $A$. 
For any (full-ranked) synthetic node features $X'\in\mathbb{R}^{c\times d}$, there exists a synthetic adjacency matrix $A'\in\mathbb{R}^{c\times c}_{\geq0}$ such that the gradient matching objective is achievable.
\end{proposition}

\begin{proposition}[Failed Generalization of SDC across GNNs]
\label{thm:sdc-gnns}
Consider least-squares regression with a linear GNN, $f^{C}_{W}(A, X)=C(A)XW$ parametrized by $W$; there always exists a condensed graph $\mathcal{S}^*$, such that the gradient matching objective for $f^{C}$ is achievable. However, if we train a new linear GNN $f^{C'}_{W}(A, X)$ with convolution matrix $C'(A')$ on $\mathcal{S}^*$, the relative error between the optimized model parameters of $f^{C'}_{W}$ on the real and condensed graphs is $\|W^{\mathcal{S}}_{C'}-W^{\mathcal{T}}_{C'}\|/\|W^{\mathcal{T}}_{C'}\|\geq \max\{\sigma_{\max}(Q)-1,1-\sigma_{\min}(Q)\}$, where $W^{\mathcal{T}}_{C'}=\arg\min_{W}\|\mathbf{y}-f^{C'}_{W}(A, X)\|^2_2$, $W^{\mathcal{S}}_{C'}=\arg\min_{W}\|\mathbf{y}'-f^{C'}_{W}(A', X')\|^2_2$, and $Q=\big(X^\top [C(A)]^\top [C(A)]X\big)\big(X^\top [C'(A)]^\top [C'(A)]X\big)^{-1}$.
\end{proposition}

\section{Hypergradients and Implicit Function Theorem}
\label{apd:hypergradients}

In this section, we give some additional background on the concepts of hyperparameter gradients (hypergradient for short) and the implicit function theorem (IFT).

We shall recall the notations we used in the main paper, which are summarized in~\cref{tab:notations}.

\begin{table}[!htbp]
\adjustbox{max width=\textwidth}{%
\renewcommand*{\arraystretch}{1.2}
\begin{tabular}{r|l}
\toprule
$\lambda, \theta$                            & Hyperparameters and NN parameters         \\
$\mathcal{T}, \mathcal{S}$                   & The original and synthetic datasets       \\
$\mathcal{L}^{\mathrm{train}}_{\mathcal{T}}, \mathcal{L}^{\mathrm{val}}_{\mathcal{T}}$ & The training and validation loss on the original dataset \\
$\mathcal{L}^{\mathrm{train}}_{\mathcal{S}}, \mathcal{L}^{\mathrm{val}}_{\mathcal{S}}$ & The training and validation loss on the synthetic dataset \\
$\theta^{\mathcal{T}}(\lambda):=\arg \min _\theta \mathcal{L}_{\mathcal{T}}^{\text {train }}(\theta, \lambda)$ & The optimized parameters on the original dataset, as a function of the hyperparameter \\
$\theta^{\mathcal{S}}(\lambda):=\arg \min _\theta \mathcal{L}_{\mathcal{S}}^{\text {train }}(\theta, \lambda)$ & The optimized parameters on the synthetic dataset, as a function of the hyperparameter \\
$\mathcal{L}_{\mathcal{T}}^*(\lambda):=\mathcal{L}_{\mathcal{T}}^{\mathrm{val}}\left(\theta^{\mathcal{T}}(\lambda), \lambda\right)$ & The validation loss on the original dataset, as a function of the hyperparameter $\lambda$ \\
$\mathcal{L}_{\mathcal{S}}^*(\lambda):=\mathcal{L}_{\mathcal{S}}^{\mathrm{val}}\left(\theta^{\mathcal{S}}(\lambda), \lambda\right)$ & The validation loss on the synthetic dataset, as a function of the hyperparameter $\lambda$ \\\bottomrule
\end{tabular}
}
\caption{\label{tab:notations}Notations.}
\end{table}

In~\cref{fig:hyper-gradients}, we visualize the geometrical process to define the hypergradients. First, in (a), we can plot the training loss $\mathcal{L}^{\text{train}}_{\mathcal{T}}(\theta,\lambda)$ as a two-variable function on the parameters $\theta$ and hyperparameter $\lambda$. For each hyperparameter $\lambda\in\Lambda$ (which is assumed to be a continuous interval in this case), we shall optimize the training loss to find an optimal model parameter, which forms the \textcolor{tab:blue}{blue} curve in $(\theta,\lambda)$-plane in (a). Now, we shall substitute the fitted parameter $\theta^{\mathcal{T}}(\lambda):=\arg \min _\theta \mathcal{L}_{\mathcal{T}}^{\text {train }}(\theta, \lambda)$ as a function of the hyperparameter into the validation loss $\mathcal{L}^{\mathrm{val}}_{\mathcal{T}}(\theta,\lambda)$. Geometrically, it is projecting the \textcolor{tab:blue}{blue} implicit optimal parameter curve onto the surface of the validation loss, as shown in (b). The projected \textcolor{tab:orange}{orange} curve is the validation loss as a one-variable function of the hyperparameter lambda $\mathcal{L}_{\mathcal{T}}^*(\lambda):=\mathcal{L}_{\mathcal{T}}^{\mathrm{val}}\left(\theta^{\mathcal{T}}(\lambda), \lambda\right)$. Finally, the \textcolor{tab:purple}{purple} curve represents the hyperparameter gradients, which is the slope of the tangent line on the \textcolor{tab:orange}{orange} validation loss curve.

\begin{figure}[!htbp]
\begin{floatrow}
\ffigbox[\textwidth]{%
    \begin{subfloatrow}[2]
        \ffigbox[0.5\columnwidth]{%
            \includegraphics[width=\linewidth]{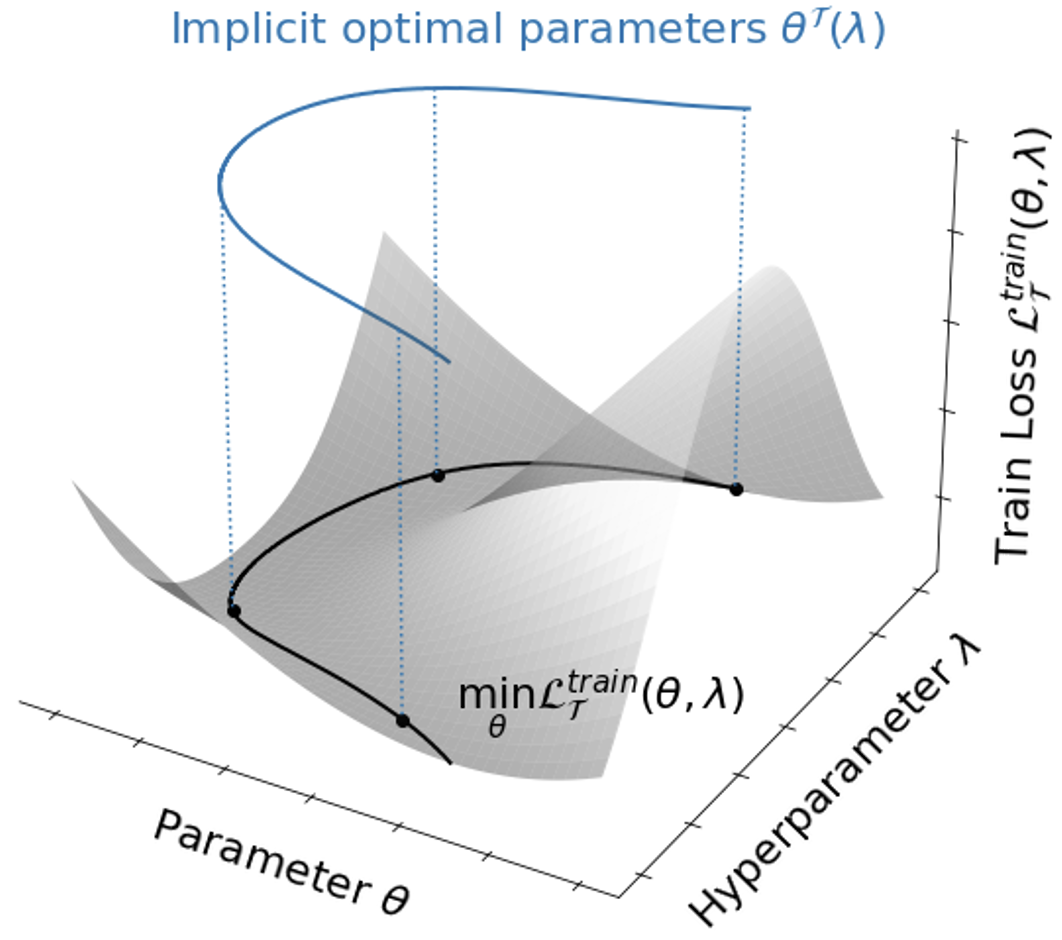}}{\caption{\label{fig:hyper-gradients-train}Train loss $\mathcal{L}^{\mathrm{train}}_{\mathcal{T}}$}}
        \ffigbox[0.5\columnwidth]{%
            \includegraphics[width=\linewidth]{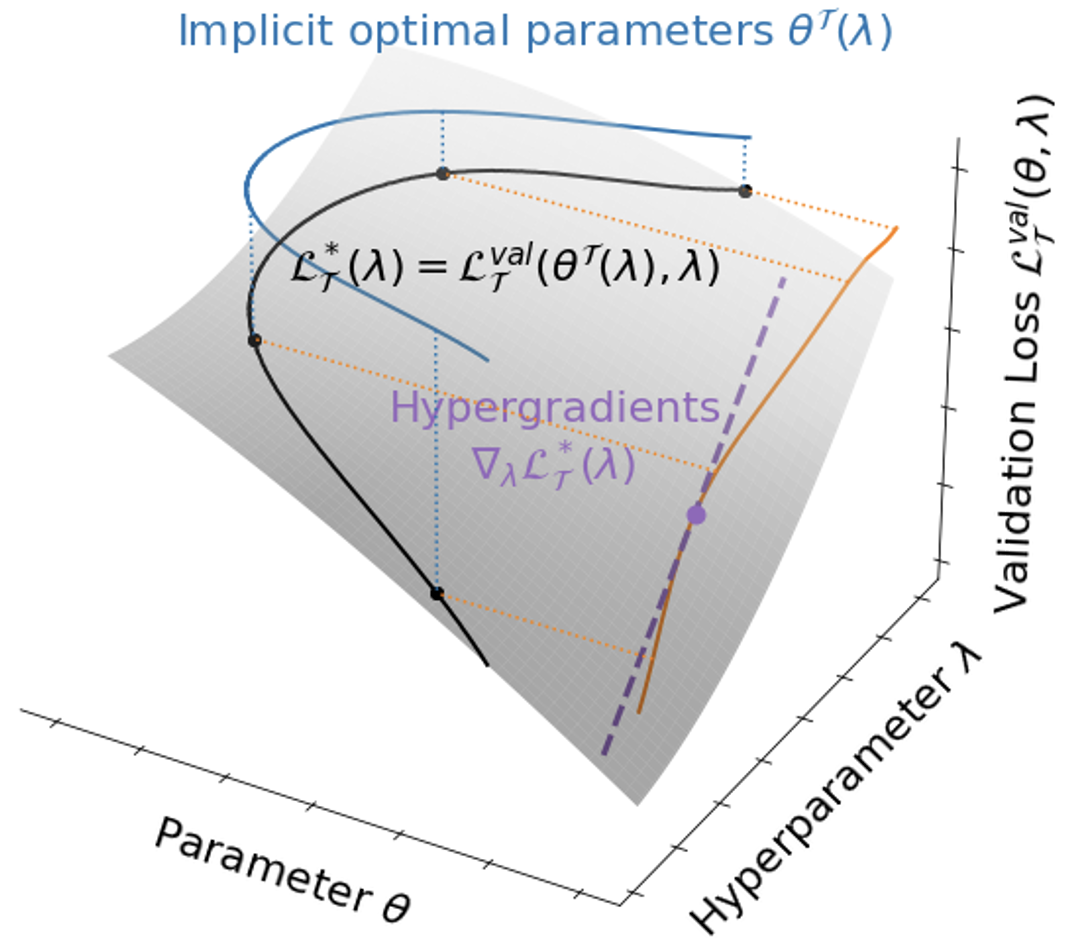}}{\caption{\label{fig:hyper-gradients-validation}Validation loss $\mathcal{L}^{\mathrm{val}}_{\mathcal{T}}$}}
    \end{subfloatrow}
}{%
    \caption{\label{fig:hyper-gradients}
    Loss landscape w.r.t. $\theta$ and $\lambda$. A hyperparameter $\lambda$ has an optimal parameter $\theta^{\mathcal{T}}(\lambda)$ \big(\textcolor{tab:blue}{blue} curve in $(\theta,\lambda)$-plane in \textbf{(a)}\big) that minimizes the train loss. In \textbf{(b)}, injecting optimal parameters $\theta^{\mathcal{T}}(\lambda)$ into the validation loss, we obtain a function of validation loss w.r.t. $\lambda$ \big(denoted as $\mathcal{L}^\star_{\mathcal{T}}(\lambda)
    $\big) in $(\mathcal{L}, \lambda)$-plane, shown as the \textcolor{tab:orange}{orange} curve. 
    The \textcolor{tab:purple}{purple} dash line illustrates the hypergradients, i.e., gradient of $\mathcal{L}_{\mathcal{T}}^\star(\lambda) 
    $ w.r.t. $\lambda$.}%
}
\end{floatrow}
\end{figure}
\section{More Theoretical Results}
\label{apd:proofs}

In this section, we provide the proofs to the theoretical results \cref{thm:sdc-valid,thm:sdc-1dcnns,thm:sdc-overfit,thm:sdc-gnns} and \cref{thm:equi-hga-hc}, together with some extended theoretical discussions, including generalizing the \emph{linear convolution regression problem} to non-convex losses and non-linear models (see~\cref{apd:proofs:non-convexlinear}).

To proceed, please recall the \emph{linear convolution regression problem} defined in~\cref{apd:prelim:linear}, the achievability of gradient-matching objective (\cref{eq:sdc-gm}) defined as~\cref{assum:achievability} in~\cref{apd:problem}.

\subsection{Validity of Standard Dataset Condensation}
\label{apd:proofs:sdc-valid}

As the first step, we verify the validity of the standard dataset condensation (SDC) using the gradient-matching objective~\cref{eq:sdc-gm} for the \emph{linear convolution regression problem}.

\begin{lemma}(Validity of SDC)
\label{thm:sdc-valid}
Consider least square regression with linear convolution model $f_{W}(A,X)=C(A)XW$ parameterized by $W$. If the gradient-matching objective of $f_{W}$ is achievable, then the optimizer on the condensed dataset $\mathcal{S}$, i.e., $W^\mathcal{S}=\arg\min_{W}\mathcal{L}_{\mathcal{S}}(W)$ is also optimal for the original dataset, i.e., $\mathcal{L}_{\mathcal{T}}(W^\mathcal{S})=\min_{W}\mathcal{L}_{\mathcal{T}}(W)$.
\end{lemma}

\begin{proof}
In the \emph{linear convolution regression problem}, sum-of-squares loss is used, and $\mathcal{L}_{\mathcal{T}}(W)=\|CXW-\mathbf{y}\|_2^2$ (similarly $\mathcal{L}_{\mathcal{S}}(W)=\|C'X'W-\mathbf{y}'\|_2^2$ where $C'=C(A')$). We assume $X^\top C^\top CX\in\mathbb{R}^{d\times d}$ is invertible and we can apply the optimizer formula for ordinary least square (OLS) regression to find the optimizer $W^\mathcal{T}$ of $\mathcal{L}_{\mathcal{T}}(W)$ as,
\[
    W^\mathcal{T}=(X^\top C^\top CX)^{-1}X^\top C^\top\mathbf{y}.
\]
Also, we can compute the gradients of $\mathcal{L}_{\mathcal{T}}(W)$ w.r.t $W$ as,
\[
    \nabla_{W}\mathcal{L}_{\mathcal{T}}(W)=2X^\top C^\top (CXW-\mathbf{y}),
\]
and similarly for $\nabla_{W}\mathcal{L}_{\mathcal{S}}(W)$.

Given the achievability of the gradient-matching objective of $f_{W}$, we know there exists a non-degenerate trajectory $(W^{\mathcal{S}}_t)_{t=0}^{T-1}$ which spans the entire parameter space, i.e., $\text{span}(W^{\mathcal{S}}_0,\ldots,W^{\mathcal{S}}_{T-1})=\mathbb{R}^{d}$, such that 
the gradient-matching loss ({the objective of} \cref{eq:sdc-gm} {without expectation}) on this trajectory is $0$. Assuming $D(\cdot,\cdot)$ is the $L_2$ norm~\citep{zhao2020dataset}, this means,
\[
    \nabla_{W}\mathcal{L}_{\mathcal{T}}(W^\mathcal{S}_t)=\nabla_{W}\mathcal{L}_{\mathcal{S}}(W^\mathcal{S}_t)\quad\text{for}\quad t\in[T].
\]
Substitute in the formula for the gradients $\nabla_{W}\mathcal{L}_{\mathcal{T}}(W)$ and  $\nabla_{W}\mathcal{L}_{\mathcal{S}}(W)$, we then have,
\[
    X^\top C^\top (CXW^\mathcal{S}_t-\mathbf{y})=X'^\top C'^\top (C'X'W^\mathcal{S}_t-\mathbf{y}')\quad\text{for}\quad t\in[T].
\]
Since the set of $\{W^\mathcal{S}_t\}_{t=0}^{T-1}$ spans the complete parameter space $\mathbb{R}^{d}$, we can transform the set of vectors $\{\omega_t\cdot W^\mathcal{S}_t\}_{t=0}^{T-1}$ to the set of unit vectors $\{\mathbf{e}^d_i\}_{i=0}^{d-1}\in\mathbb{R}^d$ by a linear transformation. Meanwhile, the set of $T$ equations above can be transformed to,
\[
    X^\top C^\top (CX\mathbf{e}^d_i-\mathbf{y})=X'^\top C'^\top (C'X'\mathbf{e}^d_i-\mathbf{y}')\quad\text{for}\quad i\in[d].
\]
This directly leads to $X^\top C^\top CX=X'^\top C'^\top C'X'$ and $X^\top C^\top\mathbf{y}=X'^\top C'^\top\mathbf{y}'$.

Using the formula for the optimizers $W^\mathcal{T}$ and $W^\mathcal{S}$ above, we readily get,
\[
    W^\mathcal{T}=(X^\top C^\top CX)^{-1}X^\top C^\top\mathbf{y}=(X'^\top C'^\top C'X')^{-1}X'^\top C'^\top\mathbf{y}'=W^\mathcal{S}.
\]
And hence,
\[
    \mathcal{L}_{\mathcal{T}}(W^\mathcal{S})=\mathcal{L}_{\mathcal{T}}(W^\mathcal{T})=\min_{W}\mathcal{L}_{\mathcal{T}}(W),
\]
which concludes the proof.
\end{proof}

Despite its simplicity, \cref{thm:sdc-valid} directly verifies the validity of the gradient-matching formulation of standard dataset condensation on some specific learning problems. Although the gradient-matching formulation (\cref{eq:sdc-gm}) is an efficient but weaker formulation than the bilevel formulation of SDC (\cref{eq:sdc-bl}), we see it is strong enough for some of the \emph{linear convolution regression problem}. 

\subsection{Generalization Issues of SDC}
\label{apd:proofs:sdc-issue}

Now, we move forward and focus on the generalization issues of (the gradient-matching formulation of) the standard dataset condensation (SDC) across GNNs.

First, we prove the successful generalization of SDC across 1D-CNNs as follows, which is very similar to the proof of~\cref{thm:sdc-valid}.

\begin{prevproof}{thm:sdc-1dcnns}
In~\cref{thm:sdc-1dcnns}, we consider one-dimensional linear convolution models $f^{2K+1}(X)=(\sum_{k=-K}^{k=K}\alpha^{(k)} P^k)XW$ parameterized by $\alpha\in\mathbb{R}^{p}$ and $W\in\mathbb{R}^{d}$ (where $p=2K+1$). If we denote,
\[
    C=\sum_{k=-K}^{k=K}\alpha^{(k)} P^k\quad\text{and}\quad\theta=[\alpha,W]\in\mathbb{R}^{p+d}
\]
then from the proof of~\cref{thm:sdc-valid} we know the gradients of $\mathcal{L}_{\mathcal{T}}(W)$ w.r.t $W$ is again,
\[
    \nabla_{W}\mathcal{L}_{\mathcal{T}}(W)=2X^\top C^\top (CXW-\mathbf{y}).
\]

We know the achievability of the gradient-matching objective means there exists a non-degenerate trajectory $(\theta^{\mathcal{S}}_t)_{t=0}^{T-1}$ which spans the entire parameter space, i.e., $\text{span}(\theta^{\mathcal{S}}_0,\ldots,\theta^{\mathcal{S}}_{T-1})=\mathbb{R}^{p+d}$. By decomposing $\theta^{\mathcal{S}}_t$ into $[\alpha^{\mathcal{S}}_t, W^{\mathcal{S}}_t]$, we know that there exists $(\alpha^{\mathcal{S}}_t)_{t=0}^{T-1}$ which spans $\mathbb{R}^{p}$ and there exists $(W^{\mathcal{S}}_t)_{t=0}^{T-1}$ which spans $\mathbb{R}^{d}$.

Since the gradient-matching objective is minimized to $0$ on $(W^{\mathcal{S}}_t)_{t=0}^{T-1}$ which spans $\mathbb{R}^{d}$, following the same procedure as the proof of~\cref{thm:sdc-valid}, we again obtain,
\[
    X^\top C^\top \mathbf{y}=X'^\top C'^\top \mathbf{y}'.
\]

Meanwhile, since the same gradient-matching objective is also minimized to $0$ on $(\alpha^{\mathcal{S}}_t)_{t=0}^{T-1}$ which spans $\mathbb{R}^{p}$, we have,
\[
    X^\top\Big(\sum_{k=-K}^{k=K}(\alpha^{\mathcal{S}}_t)^{(k)} P^k\Big)^\top \mathbf{y}=X'^\top \Big(\sum_{k=-K}^{k=K}(\alpha^{\mathcal{S}}_t)^{(k)} P'^k\Big)^\top \mathbf{y}'\quad\text{for}\quad t\in[T].
\]
Again by linear combining the above $T$ equations and because $(\alpha^{\mathcal{S}}_t)_{t=0}^{T-1}$ can be transformed to the unit vectors in $\mathbb{R}^p$, we have,
\[
    X^\top \big(P^k\big)^\top \mathbf{y}=X'^\top \big(P'^k\big)^\top \mathbf{y}'\quad\text{for}\quad k=-K,\ldots,K.
\]
Hence, for any new trajectory $(\alpha'^{\mathcal{S}}_t)_{t=0}^{T-1}$ which spans $\mathbb{R}^{p'}$ where $p'=2K'+1$, by linear combining the above equations, we have,
\[
    X^\top\Big(\sum_{k=-K}^{k=K}(\alpha'^{\mathcal{S}}_t)^{(k)} P^k\Big)^\top \mathbf{y}=X'^\top \Big(\sum_{k=-K}^{k=K}(\alpha'^{\mathcal{S}}_t)^{(k)} P'^k\Big)^\top \mathbf{y}'\quad\text{for}\quad t\in[T'].
\]
With similar procedure for the $X^\top C^\top CX$ part, we conclude that on the new trajectory $(\theta'^{\mathcal{S}}_t)_{t=0}^{T-1}$ 
\[
    \nabla_{W}\mathcal{L}_{\mathcal{T}}(\alpha,W)=\nabla_{W}\mathcal{L}_{\mathcal{S}}(\alpha,W).
\]

It remains to prove that on any new trajectory $\nabla_{\alpha}\mathcal{L}_{\mathcal{T}}(\alpha,W)=\nabla_{\alpha}\mathcal{L}_{\mathcal{S}}(\alpha,W)$. Only need to note that,
\[
    \nabla_{\alpha^{(k)}}\mathcal{L}_{\mathcal{T}}(\alpha,W)=2W^\top X^\top P^k(CXW-\mathbf{y}).
\]
Hence, by the $p$ equations above, we can readily show,
\[
    X^\top P^k \mathbf{y} = X'^\top P'^k \mathbf{y}'\quad\text{for}\quad k=-K,\ldots,K.
\]
Again with a similar procedure for the $X^\top C^\top CX$ part, we finally can show that on the new trajectory $(\theta'^{\mathcal{S}}_t)_{t=0}^{T-1}$ 
\[
    \nabla_{\alpha}\mathcal{L}_{\mathcal{T}}(\alpha,W)=\nabla_{\alpha}\mathcal{L}_{\mathcal{S}}(\alpha,W).
\]
This concludes the proof.
\end{prevproof}

Then we focus on the linear GNNs, we want to verify the insight that the learned adjacency $A'$ of the condensed graph has ``too many degrees of freedom'' so that can easily overfit the gradient-matching objective, no matter what learned synthetic features $X'$ are. Again, the proof of~\cref{thm:sdc-overfit} uses some results in the proof of~\cref{thm:sdc-valid}.

\begin{prevproof}{thm:sdc-overfit}
Now, we consider a linear GNN defined as $f(A,X)=C(A)XW$. From the proof of~\cref{thm:sdc-valid}, we know that for the gradient-matching objective of $f$ to be achievable, it is equivalent to require that,
\[
    X^\top C^\top CX=X'^\top C'^\top C'X'\quad\text{and}\quad X^\top C^\top \mathbf{y}=X'^\top C'^\top \mathbf{y}',
\]
where $C$ and $C'$ refer to $C(A)$ and $C(A')$ respectively.

Firstly we note that once we find $C'$ and $X'$ such that satisfy the first condition $X^\top C^\top CX=X'^\top C'^\top C'X'$, we can always find $\mathbf{y}'\in\mathbb{R}^c$ such that $X^\top C^\top \mathbf{y}=X'^\top C'^\top \mathbf{y}'$ since $X^\top C^\top \mathbf{y}\in\mathbb{R}$ is a scalar.

Now, we focus on finding the convolution matrix $C'$ and the node feature matrix $X'$ of the condensed synthetic graph to satisfy $X^\top C^\top CX=X'^\top C'^\top C'X'$. We assume $n\gg c\gg d$ and consider the diagonalization of $X^\top C^\top CX\in\mathbb{R}^{d\times d}$. Since $X^\top C^\top CX$ is positive semi-definite, it can be diagonalized as $X^\top C^\top CX=VS^2V^\top$ where $V\in\mathbb{R}^{d}$ is an orthogonal matrix and $S\in\mathbb{R}^{d}$ is a diagonal matrix that $S=\diag{s_1,\ldots,s_d}$.

For any (real) semi-unitary matrix $U\in\mathbb{R}^{c\times d}$ such that $U^\top U=I_d$, we can construct $C'X'=USV^\top\in\mathbb{R}^{c\times d}$ and we can easily verify they satisfy the condition,
\[
    X'^\top C'^\top C'X' = VSU^\top USV^\top=VS^2V^\top=X^\top C^\top CX.
\]

Then since $X'$ is full ranked, for any $X'$, by considering the singular-value decomposition of $X'$, we see that we can always find a convolution matrix $C'$ such that $C'X'=USV$ and this concludes the proof.
\end{prevproof}

Finally, we use some results of~\cref{thm:sdc-overfit} to prove~\cref{thm:sdc-gnns}, the failure of SDC when generalizing across GNNs.

\begin{prevproof}{thm:sdc-gnns}
We prove this by two steps. 

For the first step, we aim to show that there always exist a condensed synthetic dataset $\mathcal{S}$ such that achieves the gradient-matching objective but the learned adjacency matrix $A'=I_c$ is the identity matrix. Clearly this directly follows form the proof of~\cref{thm:sdc-overfit}, where we only require $C'X'=USV$ (see the proof of~\cref{thm:sdc-overfit} for details). If the learned adjacency matrix $A'=I_c$, the for any GNNs, the corresponding convolution matrix $C'$ is also (or proportional to) identity, thus we only need to set the learned node feature matrix $X'=USV$ to satisfy the condition. The first step is proved.

For the second step, we evaluate the relative estimation error of the optimal parameter when transferred to a new GNN $f^{\mathscr{C}}_{W}$ with convolution filter $\mathscr{C}(\cdot)$, i.e., $\|W^{\mathcal{S}}_{\mathscr{C}}-W^{\mathcal{T}}_{\mathscr{C}}\|/\|W^{\mathcal{T}}_{\mathscr{C}}\|$. Using the formula for the optimal parameter in the proof of~\cref{thm:sdc-valid} again, we have,
\[
    W^\mathcal{T}_\mathscr{C}=(X^\top \mathscr{C}^\top \mathscr{C}X)^{-1}X^\top \mathscr{C}^\top\mathbf{y},
\]
and
\[
    W^\mathcal{S}_\mathscr{C}=(X'^\top \mathscr{C}'^\top \mathscr{C}'X)^{-1}X'^\top \mathscr{C}'^\top\mathbf{y}',
\]
where $\mathscr{C}'=\mathscr{C}(A')=\mathscr{C}(I_c)=C(I_c)$ (the last equation use the fact that the convolution matrix of GNNs are the same if the underlying graph is identity).

Moreover, by the validity of SDC on $f^{C}_W$, we know, (see the proof of~\cref{thm:sdc-valid} for details),
\[
    X'^\top C'^\top C'X'=X^\top C^\top CX\quad\text{and}\quad X'^\top C'^\top\mathbf{y}'=X^\top C^\top \mathbf{y}
\]
Thus, altogether we derive that $X'^\top \mathscr{C}'^\top \mathscr{C}'X = X^\top C^\top CX$ and $X'^\top \mathscr{C}'^\top \mathbf{y}' = X^\top C^\top \mathbf{y}$. And therefore,
\[
    W^\mathcal{S}_\mathscr{C}=(X^\top C^\top CX)^{-1}X^\top C^\top\mathbf{y}.
\]
Now, note that,
\[
\begin{split}
    &\|W^{\mathcal{S}}_{\mathscr{C}}-W^{\mathcal{T}}_{\mathscr{C}}\|/\|W^{\mathcal{T}}_{\mathscr{C}}\|\\
    =&\Big\|\Big(\big(X^\top [C(A)]^\top [C(A)]X\big)\big(X^\top [\mathscr{C}(A)]^\top [\mathscr{C}(A)]X\big)^{-1})-I_d\Big)X^\top C^\top\mathbf{y}\Big\|/\|X^\top C^\top\mathbf{y}\|\\
    \geq&\max\{\sigma_{\max}(Q)-1,1-\sigma_{\min}(Q)\} \\
\end{split}
\]
where $Q=\big(X^\top [C(A)]^\top [C(A)]X\big)\big(X^\top [\mathscr{C}(A)]^\top [\mathscr{C}(A)]X\big)^{-1}$. This concludes the proof.
\end{prevproof}

\subsection{Validity of HCDC}
\label{apd:proofs:hcdc-issue}

Finally, we complete the proof of~\cref{thm:equi-hga-hc} with more detials.

\begin{prevproof}{thm:equi-hga-hc}
Firstly, we prove the \emph{necessity} by contradiction. 

If there exists $\lambda_0\in\tilde{\Lambda}$ s.t. the two gradient vectors are not aligned at $\lambda_0$, then there exists small perturbation $\Delta\lambda_0$ such that $\mathcal{L}_{\mathcal{T}}^*(\lambda_0+\Delta\lambda_0)-\mathcal{L}_{\mathcal{T}}^*(\lambda_0)$ and $\mathcal{L}_{\mathcal{S}}^*(\lambda_0+\Delta\lambda_0)-\mathcal{L}_{\mathcal{S}}^*(\lambda_0)$ have different signs.

Secondly, we prove the \emph{sufficiency} by path-integration. 

For any pair $\lambda_1\neq\lambda_2\in\tilde{\Lambda}$, we have a path $\gamma(\lambda_1,\lambda_2)\in\tilde{\Lambda}$ from $\lambda_2$ and $\lambda_1$, then integrating hypergradients $\nabla_\lambda\mathcal{L}_{\mathcal{T}}^*(\lambda)$ along the path recovers the hyperparameter-calibration condition.
More specifically, along the path we have $\mathcal{L}_{\mathcal{T}}^*(\lambda_1)-\mathcal{L}_{\mathcal{T}}^*(\lambda_2)=\int_{\gamma(\lambda_1,\lambda_2)}\nabla_\lambda\mathcal{L}_{\mathcal{T}}^*(\lambda)\mathrm{d}\lambda$ (similar for $\nabla_\lambda\mathcal{L}_{\mathcal{S}}^*(\lambda)$).
Thus we have,
\[
\begin{split}
     &(\mathcal{L}_{\mathcal{T}}^*(\lambda_1)-\mathcal{L}_{\mathcal{T}}^*(\lambda_2))(\mathcal{L}_{\mathcal{S}}^*(\lambda_1)-\mathcal{L}_{\mathcal{S}}^*(\lambda_2))\\
     =&\Big(\int_{\gamma(\lambda_1,\lambda_2)}\nabla_\lambda\mathcal{L}_{\mathcal{T}}^*(\lambda)\mathrm{d}\lambda\Big)\Big(\int_{\gamma(\lambda_1,\lambda_2)}\nabla_\lambda\mathcal{L}_{\mathcal{S}}^*(\lambda)\mathrm{d}\lambda\Big)\\
     \geq&\int_{\gamma(\lambda_1,\lambda_2)}\big\langle\sqrt{\nabla_\lambda\mathcal{L}_{\mathcal{T}}^*(\lambda)},\sqrt{\nabla_\lambda\mathcal{L}_{\mathcal{S}}^*(\lambda)}\big\rangle\mathrm{d}\lambda\\
     \geq&0
\end{split}
\]
the second last inequality by Cauchy-Schwarz inquality and the last inequality by $\cos(\nabla_\lambda\mathcal{L}_{\mathcal{T}}^*(\lambda),\nabla_\lambda\mathcal{L}_{\mathcal{T}}^*(\lambda))=0$ for any $\lambda\in\gamma(\lambda_1,\lambda_2)\in\tilde{\Lambda}$.

This concludes the proof.
\end{prevproof}

\subsection{Generalize to Non-Convex and Non-Linear Case}
\label{apd:proofs:non-convexlinear}

Although the results above are obtained for least squares loss and linear convolution model, it \emph{still reflects the nature of general non-convex losses and non-linear models}.
Since dataset condensation is effectively matching the local minima $\{\theta^{\mathcal{T}}\}$ of the original loss $\mathcal{L}^{train}_{\mathcal{T}}(\theta,\psi)$ with the local minima $\{\theta^{\mathcal{S}}\}$ of the condensed loss $\mathcal{L}^{train}_{\mathcal{S}}(\theta,\psi)$, within the small neighborhoods surrounding the pair of local minima $(\theta^{\mathcal{T}},\theta^{\mathcal{S}})$, we can approximate the non-convex loss and non-linear model with a convex/linear one respectively.
Hence the generalizability issues with convex loss and liner model may hold.

\section{\ouralgo on Discrete and Continuous Search Spaces}
\label{apd:implementations}

In this subsection, we illustrate how to tackle the two types of search spaces: (1) discrete and finite $\Lambda$ and (2) continuous and bounded $\Lambda$, respectively, illustrated by the practical search spaces of ResNets on images and Graph Neural Networks (GNNs) on graphs.

\para{Discrete search space of neural architectures on images.}
Typically the neural architecture search (NAS) problem aims to find the optimal neural network architecture with the best validation performance on a dataset from a large set of candidate architectures. 
One may simply train the set of $p$ pre-defined architectures $\{f^{(i)} \mid i=1,\ldots,p\}$ and rank their validation performance.
We can transform this problem as a continuous HPO, by defining an ``interpolated'' model (i.e., a super-net~\citep{wang2020rethinking}) $f^{\lambda}_\theta$, where hyperparameters $\lambda=[\lambda^{(1)},\ldots,\lambda^{(p)}]\in\Lambda$ and $\theta$ is the model parameters.
This technique is known as the differentiable NAS approach (see~\cref{apd:extend-related}), e.g., DARTS~\citep{liu2018darts}, which usually follows a cell-based search space~\citep{zoph2018learning} and continuously relaxes the original discrete search space.
Subsequently, in~\citep{xie2018snas,dong2019searching}, the Gumbel-softmax trick~\citep{jang2016categorical,maddison2016concrete} is applied to approximate the hyperparameter gradients.
We apply the optimization strategy in GDAS~\citep{dong2019searching}, which also provides the approximations of the hypergradients in our experiments.

\para{Continuous search space of graph convolution filters.}
Many graph neural networks (GNNs) can be interpreted as performing message passing on node features, followed by feature transformation and an activation function.
In this regard, these GNNs differ from each other by choice of convolution matrix (see~\cref{apd:problem} for details).
We consider the problem of searching for the best convolution filter of GNNs on a graph dataset.
In~\cref{apd:problem,apd:proofs}, we theoretically justify that dataset condensation for HPO is challenging on graphs due to overfitting issues. 
Nonetheless, this obstacle is solved by~\ouralgo.

A natural continuous search space of convolution filters often exists in GNNs, e.g., when the candidate convolution filters can be expressed by a generic formula.
For example, we make use of a truncated powers series of the graph Laplacian matrix to model a wide range of convolution filters, as considered in ChebNet~\citep{defferrard2016convolutional} or SIGN~\citep{frasca2020sign}.
Given the differentiable generic formula of the convolution filters, we can treat the convolution filter (or, more specifically, the parameters in the generic formula) as hyperparameters and evaluate the hypergradients using the implicit differentiation methods discussed in \cref{ssec:hypergrad-evaluation}.

\section{Efficient Design of HCDC Algorithm}
\label{sapd:efficient}

For continuous hyperparameters, $\Lambda$ itself is usually compact and connected, and the minimal extended search space is $\tilde{\Lambda}=\Lambda$.
Consider a discrete search space $\Lambda$, which consists of $p$ candidate hyperparameters.
We can naively construct $\tilde{\Lambda}$ as $O(p^2)$ continuous paths connecting pairs of candidate hyperparameters (see~\cref{fig:trajectory-matching-discrete} in~\cref{sapd:efficient} for illustration). 
This is undesirable due to its quadratic complexity in $p$; for example, the number of candidate architectures $p$ is often a large number for practical NAS problems.

We propose a construction of $\tilde{\Lambda}$ with linear complexity in $p$, which works as follows.
For any $i\in[p]$, we construct a ``representative'' path, named $i$-th HPO trajectory, which starts from $\lambda^\mathcal{S}_{i,0}=\lambda_i\in\Lambda$ and updates through $\lambda^\mathcal{S}_{i,t+1}\leftarrow\lambda^\mathcal{S}_{i,t}-\eta\nabla_\lambda\mathcal{L}_{\mathcal{S}}^*(\lambda^\mathcal{S}_{i,t})$, shown as the \textcolor{tab:orange}{orange} dashed lines in~\cref{fig:trajectory-matching-discrete}.
We assume all of the $p$ trajectories will approach the same or equivalent optima $\lambda^{\mathcal{S}}$, forming ``connected'' paths (i.e., \textcolor{tab:orange}{orange} dashed lines which merge at the optima $\lambda^{\mathcal{S}}$) between any pair of hyperparameters $\lambda_i\neq\lambda_j\in\Lambda$. 
This construction is also used in a continuous search space (as shown in \cref{fig:trajectory-matching-continuous}) to save computation (except that we have to select the starting randomly points $\lambda_i\sim\mathbb{P}_{\Lambda}$).

\begin{figure}[!htbp]
\begin{floatrow}
\ffigbox[\textwidth]{%
    \begin{subfloatrow}[2]
        \ffigbox[0.45\columnwidth]{%
            \includegraphics[width=\linewidth]{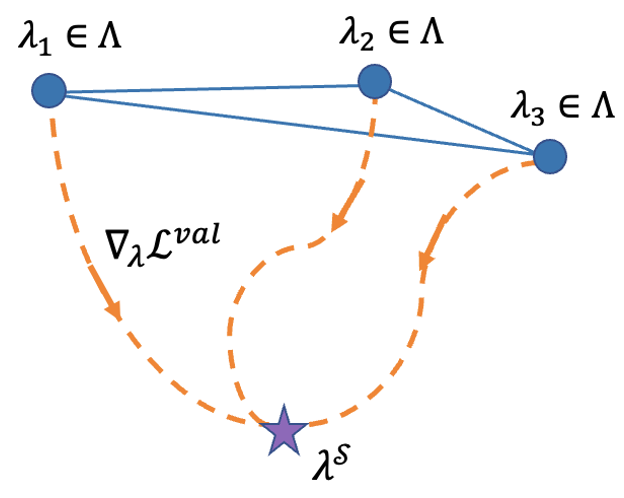}}{
           \vspace{-2em} \caption{\label{fig:trajectory-matching-discrete}Discrete $\Lambda=\{\lambda_i\}$}}
            \hfill
        \ffigbox[0.45\columnwidth]{%
            \includegraphics[width=\linewidth]{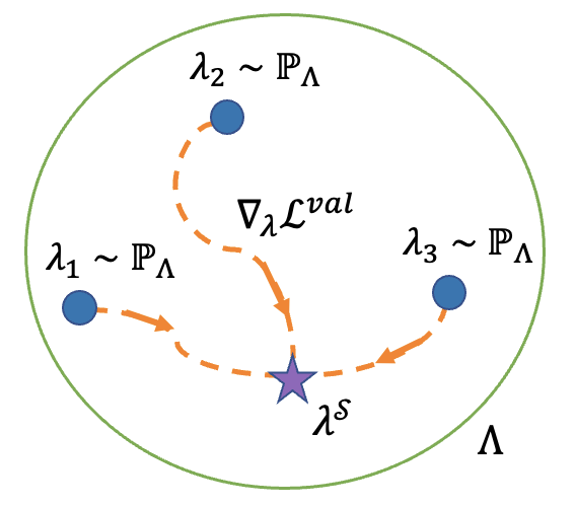}}{
           \vspace{-2em} \caption{\label{fig:trajectory-matching-continuous}Continuous $\Lambda\ni\lambda^{\mathcal{S}}$}}
    \end{subfloatrow}
}{%
    \caption{\label{fig:trajectory-matching}
    Illustration of the constructed extended search space $\tilde{\Lambda}$ illustrated as the \textcolor{tab:orange}{orange} trajectory for both \textbf{a)} discrete $\Lambda$ and \textbf{(b)} continuous $\Lambda$.
    The trajectory starts from $\lambda^\mathcal{S}_{i,0}=\lambda_i\in\Lambda$ for discrete $\Lambda$ (or random points for continuous $\Lambda$), and updates through $\lambda^\mathcal{S}_{i,t+1}\leftarrow\lambda^\mathcal{S}_{i,t}-\eta\nabla_\lambda\mathcal{L}_{\mathcal{S}}^*(\lambda^\mathcal{S}_{i,t})$.}%
}
\end{floatrow}
\end{figure}

\section{Complexity Analysis of HCDC Algorithm}
\label{apd:complexity-analysis}

In this section, we provide additional details on the complexity analysis of the HCDC algorithm~\cref{alg:hcdc}. Following the algorithm pseudocode, we consider the discrete hyperparameter search space of size $p$. The overall time complexity of HCDC is proportional to this size $p$ of the hyperparameter search space.

If we denote the dimensionality of the model parameters $\theta$ and hyperparameters $\lambda$ by $P$ and $H$ respectively. We know the time complexity of the common model parameter update is $O(P)$. Based on~\citep{lorraine2020optimizing}, the hyperparameter update~\cref{alg:hcdc:update-hyperparameters} needs $O(P+H)$ time (since we fixed the truncated power of the Neumann series approximation as constant) and $O(P+H)$ memory. In~\cref{alg:hcdc:update-synthetic-set}, we need another backpropagation to take gradients of $\nabla_\lambda\mathcal{L}^{\text{val}}_{\mathcal{S}}(\theta,\lambda)$ w.r.t. $\mathcal{S}^{\text{val}}$. This update is performed in a mini-batch manner and supposes the dimensionality of validation samples in the mini-batch is $B$, the~\cref{alg:hcdc:update-synthetic-set} requries $O(HB)$ time and memory. 
\section{Implementation Details}
\label{apd:implement-detials}

In this section, we list more implementation details on the experiments in~\cref{sec:experiments}.

For the synthetic experiments on CIFAR-10, we randomly split the CIFAR-10 images into $M=20$ splits and perform cross-validation. For the baseline methods (Random, SDC-GM, SDC-DM), the dataset condensation is performed independently for each split. For \ouralgo, we first condense the training set of the synthetic dataset by SDC-GM. Then, we learn a separate validation set with $1/M$-size of the training set and train with the \ouralgo objective on the $M$-HPO trajectories as described in~\cref{sec:algorithm}. We report the correlation between the ranking of splits (in terms of their validation performance on this split). For the Early-Stopping method, we only train the same number of iterations as the other methods (with the same batchsize), which means there are only $c/n*500$ epochs.

For the experiments about finding the best convolution filter on (large) graphs, we create the set of ten candidate convolution filters as (see~\cref{tab:gnns} for definitions and references) GCN, SAGE-Mean, SAGE-Max, GAT, GIN-$\epsilon$, GIN-0, SGC(K=2), SGC(K=3), ChebNet(K=2), ChebNet(K-3). The implementations are provided by PyTorch Geometric \url{https://pytorch-geometric.readthedocs.io/en/latest/modules/nn.html}. We also select the GNN width from $\{128, 256\}$ and the GNN depth from $\{2, 4\}$ so there are $10\times2\times2=40$ models in total. 

For the experiments about speeding up off-the-shelf graph architecture search algorithms, we adopt GraphNAS~\citep{gao2019graphnas} together with their proposed search space from their official repository \url{https://github.com/GraphNAS/GraphNAS}. We apply to the ogbn-arxiv graph with condensation ratio $c_{train}/n=0.5\%$.

\section{More experiments}
\label{apd:more-experiments}

\textbf{Synthetic experiments on CIFAR-10.} We first consider a synthetically created set of hyperparameters on the image dataset, CIFAR-10. Consider the $M$-fold cross-validation, where a fraction of $1/M$ samples are used as the validation split each time. The $M$-fold cross-validation process can be modeled by a set of $M$ hyperparameters $\{\varphi_i\in\{0,1\}\mid i=1,\ldots,M\}$, where $\varphi_i=1$ if and only if the $i$-th fold is used for validation. The problem of finding the best validation performance among the $M$ results can be modeled as a hyperparameter optimization problem with a discrete search space $|\Psi|=M$. We compare \ouralgo with the gradient-matching~\citep{zhao2020dataset} and distribution matching~\citep{zhao2021datasetb} baselines. We also consider a uniform random sampling baseline and an early-stopping baseline where we train only $c/n*500$ epochs but on the original dataset. The results of $M=20$ and $c/n=1\%$ is reported in~\cref{tab:cv-image}, where we see \ouralgo achieves the highest rank correlation.

\begin{table}[H]
\adjustbox{max width=0.5\textwidth}{%
{\renewcommand{\arraystretch}{1.1}%
\begin{tabular}{rcc} \toprule
                            & \multicolumn{2}{c}{Ratio ($c_{\mathrm{train}}/n$)}     \\
    Method                  & 2\%           & 4\%                           \\ \midrule
    Random                  & $-0.03$       & $0.07$                        \\
    SDC-GM                  & $0.64$        & $0.78$                        \\
    SDC-DM                  & $0.77$        & $0.86$                        \\
    Early-Stopping          & $0.11$        & $0.24$                        \\
    \ouralgo                & $0.91$        & $0.94$                        \\ \bottomrule
\end{tabular}
}}
\caption{\label{tab:cv-image}The rank correlation and validation performance on the original dataset of the $M$-fold cross-validation ranked/selected on the condensed dataset on CIFAR-10.}
\end{table}

\clearpage
\newpage

\bibliography{reference}

\begin{thebibliography}{79}
\expandafter\ifx\csname natexlab\endcsname\relax\def\natexlab#1{#1}\fi
\expandafter\ifx\csname url\endcsname\relax
  \def\url#1{\texttt{#1}}\fi
\expandafter\ifx\csname urlprefix\endcsname\relax\def\urlprefix{URL }\fi

\bibitem[{Aljundi et~al.(2019)Aljundi, Lin, Goujaud and Bengio}]{aljundi2019gradient}
\textsc{Aljundi, R.}, \textsc{Lin, M.}, \textsc{Goujaud, B.} and \textsc{Bengio, Y.} (2019).
\newblock Gradient based sample selection for online continual learning.
\newblock \textit{Advances in neural information processing systems} \textbf{32}.

\bibitem[{Anand and Huang(2018)}]{anand2018generative}
\textsc{Anand, N.} and \textsc{Huang, P.} (2018).
\newblock Generative modeling for protein structures.
\newblock \textit{Advances in neural information processing systems} \textbf{31}.

\bibitem[{Baker et~al.(2020)Baker, Braverman, Huang, Jiang, Krauthgamer and Wu}]{baker2020coresets}
\textsc{Baker, D.}, \textsc{Braverman, V.}, \textsc{Huang, L.}, \textsc{Jiang, S. H.-C.}, \textsc{Krauthgamer, R.} and \textsc{Wu, X.} (2020).
\newblock Coresets for clustering in graphs of bounded treewidth.
\newblock In \textit{International Conference on Machine Learning}. PMLR.

\bibitem[{Balcilar et~al.(2021)Balcilar, Guillaume, H{\'e}roux, Ga{\"u}z{\`e}re, Adam and Honeine}]{balcilar2021analyzing}
\textsc{Balcilar, M.}, \textsc{Guillaume, R.}, \textsc{H{\'e}roux, P.}, \textsc{Ga{\"u}z{\`e}re, B.}, \textsc{Adam, S.} and \textsc{Honeine, P.} (2021).
\newblock Analyzing the expressive power of graph neural networks in a spectral perspective.
\newblock In \textit{Proceedings of the International Conference on Learning Representations (ICLR)}.

\bibitem[{Batson et~al.(2013)Batson, Spielman, Srivastava and Teng}]{batson2013spectral}
\textsc{Batson, J.}, \textsc{Spielman, D.~A.}, \textsc{Srivastava, N.} and \textsc{Teng, S.-H.} (2013).
\newblock Spectral sparsification of graphs: theory and algorithms.
\newblock \textit{Communications of the ACM} \textbf{56} 87--94.

\bibitem[{Bengio(2000)}]{bengio2000gradient}
\textsc{Bengio, Y.} (2000).
\newblock Gradient-based optimization of hyperparameters.
\newblock \textit{Neural computation} \textbf{12} 1889--1900.

\bibitem[{Bohdal et~al.(2020)Bohdal, Yang and Hospedales}]{bohdal2020flexible}
\textsc{Bohdal, O.}, \textsc{Yang, Y.} and \textsc{Hospedales, T.} (2020).
\newblock Flexible dataset distillation: Learn labels instead of images.
\newblock \textit{arXiv preprint arXiv:2006.08572} .

\bibitem[{Borsos et~al.(2020)Borsos, Mutny and Krause}]{borsos2020coresets}
\textsc{Borsos, Z.}, \textsc{Mutny, M.} and \textsc{Krause, A.} (2020).
\newblock Coresets via bilevel optimization for continual learning and streaming.
\newblock \textit{Advances in Neural Information Processing Systems} \textbf{33} 14879--14890.

\bibitem[{Braverman et~al.(2021)Braverman, Jiang, Krauthgamer and Wu}]{braverman2021coresets}
\textsc{Braverman, V.}, \textsc{Jiang, S. H.-C.}, \textsc{Krauthgamer, R.} and \textsc{Wu, X.} (2021).
\newblock Coresets for clustering in excluded-minor graphs and beyond.
\newblock In \textit{Proceedings of the 2021 ACM-SIAM Symposium on Discrete Algorithms (SODA)}. SIAM.

\bibitem[{Cai et~al.(2020)Cai, Wang and Wang}]{cai2020graph}
\textsc{Cai, C.}, \textsc{Wang, D.} and \textsc{Wang, Y.} (2020).
\newblock Graph coarsening with neural networks.
\newblock In \textit{International Conference on Learning Representations}.

\bibitem[{Cazenavette et~al.(2022)Cazenavette, Wang, Torralba, Efros and Zhu}]{cazenavette2022dataset}
\textsc{Cazenavette, G.}, \textsc{Wang, T.}, \textsc{Torralba, A.}, \textsc{Efros, A.~A.} and \textsc{Zhu, J.-Y.} (2022).
\newblock Dataset distillation by matching training trajectories.
\newblock In \textit{Proceedings of the IEEE/CVF Conference on Computer Vision and Pattern Recognition}.

\bibitem[{Chen et~al.(2010)Chen, Welling and Smola}]{chen2010super}
\textsc{Chen, Y.}, \textsc{Welling, M.} and \textsc{Smola, A.} (2010).
\newblock Super-samples from kernel herding.
\newblock In \textit{Proceedings of the Twenty-Sixth Conference on Uncertainty in Artificial Intelligence}.

\bibitem[{Chiang et~al.(2019)Chiang, Liu, Si, Li, Bengio and Hsieh}]{chiang2019cluster}
\textsc{Chiang, W.-L.}, \textsc{Liu, X.}, \textsc{Si, S.}, \textsc{Li, Y.}, \textsc{Bengio, S.} and \textsc{Hsieh, C.-J.} (2019).
\newblock Cluster-gcn: An efficient algorithm for training deep and large graph convolutional networks.
\newblock In \textit{Proceedings of the 25th ACM SIGKDD international conference on knowledge discovery \& data mining}.

\bibitem[{Cui et~al.(2022)Cui, Wang, Si and Hsieh}]{cui2022dc}
\textsc{Cui, J.}, \textsc{Wang, R.}, \textsc{Si, S.} and \textsc{Hsieh, C.-J.} (2022).
\newblock Dc-bench: Dataset condensation benchmark.
\newblock \textit{arXiv preprint arXiv:2207.09639} .

\bibitem[{Defferrard et~al.(2016)Defferrard, Bresson and Vandergheynst}]{defferrard2016convolutional}
\textsc{Defferrard, M.}, \textsc{Bresson, X.} and \textsc{Vandergheynst, P.} (2016).
\newblock Convolutional neural networks on graphs with fast localized spectral filtering.
\newblock In \textit{Advances in neural information processing systems}, vol.~29.

\bibitem[{Ding et~al.(2021)Ding, Kong, Li, Zhu, Dickerson, Huang and Goldstein}]{ding2021vq}
\textsc{Ding, M.}, \textsc{Kong, K.}, \textsc{Li, J.}, \textsc{Zhu, C.}, \textsc{Dickerson, J.}, \textsc{Huang, F.} and \textsc{Goldstein, T.} (2021).
\newblock Vq-gnn: A universal framework to scale up graph neural networks using vector quantization.
\newblock \textit{Advances in Neural Information Processing Systems} \textbf{34} 6733--6746.

\bibitem[{Domke(2012)}]{domke2012generic}
\textsc{Domke, J.} (2012).
\newblock Generic methods for optimization-based modeling.
\newblock In \textit{Artificial Intelligence and Statistics}. PMLR.

\bibitem[{Dong and Yang(2019)}]{dong2019searching}
\textsc{Dong, X.} and \textsc{Yang, Y.} (2019).
\newblock Searching for a robust neural architecture in four gpu hours.
\newblock In \textit{Proceedings of the IEEE/CVF Conference on Computer Vision and Pattern Recognition}.

\bibitem[{Dong and Yang(2020)}]{dong2020bench}
\textsc{Dong, X.} and \textsc{Yang, Y.} (2020).
\newblock Nas-bench-201: Extending the scope of reproducible neural architecture search.
\newblock \textit{arXiv preprint arXiv:2001.00326} .

\bibitem[{Elsken et~al.(2019)Elsken, Metzen and Hutter}]{elsken2019neural}
\textsc{Elsken, T.}, \textsc{Metzen, J.~H.} and \textsc{Hutter, F.} (2019).
\newblock Neural architecture search: A survey.
\newblock \textit{The Journal of Machine Learning Research} \textbf{20} 1997--2017.

\bibitem[{Farahani and Hekmatfar(2009)}]{farahani2009facility}
\textsc{Farahani, R.~Z.} and \textsc{Hekmatfar, M.} (2009).
\newblock \textit{Facility location: concepts, models, algorithms and case studies}.
\newblock Springer Science \& Business Media.

\bibitem[{Feurer and Hutter(2019)}]{feurer2019hyperparameter}
\textsc{Feurer, M.} and \textsc{Hutter, F.} (2019).
\newblock Hyperparameter optimization.
\newblock In \textit{Automated machine learning}. Springer, Cham, 3--33.

\bibitem[{Frasca et~al.(2020)Frasca, Rossi, Eynard, Chamberlain, Bronstein and Monti}]{frasca2020sign}
\textsc{Frasca, F.}, \textsc{Rossi, E.}, \textsc{Eynard, D.}, \textsc{Chamberlain, B.}, \textsc{Bronstein, M.} and \textsc{Monti, F.} (2020).
\newblock Sign: Scalable inception graph neural networks.
\newblock \textit{arXiv preprint arXiv:2004.11198} .

\bibitem[{Gao et~al.(2019)Gao, Yang, Zhang, Zhou and Hu}]{gao2019graphnas}
\textsc{Gao, Y.}, \textsc{Yang, H.}, \textsc{Zhang, P.}, \textsc{Zhou, C.} and \textsc{Hu, Y.} (2019).
\newblock Graphnas: Graph neural architecture search with reinforcement learning.
\newblock \textit{arXiv preprint arXiv:1904.09981} .

\bibitem[{Guo et~al.(2022)Guo, Zhao and Bai}]{guo2022deepcore}
\textsc{Guo, C.}, \textsc{Zhao, B.} and \textsc{Bai, Y.} (2022).
\newblock Deepcore: A comprehensive library for coreset selection in deep learning.
\newblock \textit{arXiv preprint arXiv:2204.08499} .

\bibitem[{Hamilton et~al.(2017)Hamilton, Ying and Leskovec}]{hamilton2017inductive}
\textsc{Hamilton, W.}, \textsc{Ying, Z.} and \textsc{Leskovec, J.} (2017).
\newblock Inductive representation learning on large graphs.
\newblock \textit{Advances in neural information processing systems} \textbf{30}.

\bibitem[{Hu et~al.(2020)Hu, Fey, Zitnik, Dong, Ren, Liu, Catasta and Leskovec}]{hu2020open}
\textsc{Hu, W.}, \textsc{Fey, M.}, \textsc{Zitnik, M.}, \textsc{Dong, Y.}, \textsc{Ren, H.}, \textsc{Liu, B.}, \textsc{Catasta, M.} and \textsc{Leskovec, J.} (2020).
\newblock Open graph benchmark: Datasets for machine learning on graphs.
\newblock \textit{Advances in neural information processing systems} \textbf{33} 22118--22133.

\bibitem[{Huan et~al.(2021)Huan, Quanming and Weiwei}]{huan2021search}
\textsc{Huan, Z.}, \textsc{Quanming, Y.} and \textsc{Weiwei, T.} (2021).
\newblock Search to aggregate neighborhood for graph neural network.
\newblock In \textit{2021 IEEE 37th International Conference on Data Engineering (ICDE)}. IEEE.

\bibitem[{Huang et~al.(2021)Huang, Zhang, Xi, Liu and Zhou}]{huang2021scaling}
\textsc{Huang, Z.}, \textsc{Zhang, S.}, \textsc{Xi, C.}, \textsc{Liu, T.} and \textsc{Zhou, M.} (2021).
\newblock Scaling up graph neural networks via graph coarsening.
\newblock In \textit{Proceedings of the 27th ACM SIGKDD Conference on Knowledge Discovery \& Data Mining}.

\bibitem[{Iyer et~al.(2021)Iyer, Khargoankar, Bilmes and Asanani}]{iyer2021submodular}
\textsc{Iyer, R.}, \textsc{Khargoankar, N.}, \textsc{Bilmes, J.} and \textsc{Asanani, H.} (2021).
\newblock Submodular combinatorial information measures with applications in machine learning.
\newblock In \textit{Algorithmic Learning Theory}. PMLR.

\bibitem[{Jang et~al.(2017)Jang, Gu and Poole}]{jang2016categorical}
\textsc{Jang, E.}, \textsc{Gu, S.} and \textsc{Poole, B.} (2017).
\newblock Categorical reparameterization with gumbel-softmax.
\newblock In \textit{International Conference on Learning Representations}.

\bibitem[{Jin et~al.(2022)Jin, Tang, Jiang, Li, Zhang, Tang and Yin}]{jin2022condensing}
\textsc{Jin, W.}, \textsc{Tang, X.}, \textsc{Jiang, H.}, \textsc{Li, Z.}, \textsc{Zhang, D.}, \textsc{Tang, J.} and \textsc{Yin, B.} (2022).
\newblock Condensing graphs via one-step gradient matching.
\newblock In \textit{Proceedings of the 28th ACM SIGKDD Conference on Knowledge Discovery and Data Mining}.

\bibitem[{Jin et~al.(2021)Jin, Zhao, Zhang, Liu, Tang and Shah}]{jin2021graph}
\textsc{Jin, W.}, \textsc{Zhao, L.}, \textsc{Zhang, S.}, \textsc{Liu, Y.}, \textsc{Tang, J.} and \textsc{Shah, N.} (2021).
\newblock Graph condensation for graph neural networks.
\newblock In \textit{International Conference on Learning Representations}.

\bibitem[{Kim et~al.(2022)Kim, Kim, Oh, Yun, Song, Jeong, Ha and Song}]{kim2022dataset}
\textsc{Kim, J.-H.}, \textsc{Kim, J.}, \textsc{Oh, S.~J.}, \textsc{Yun, S.}, \textsc{Song, H.}, \textsc{Jeong, J.}, \textsc{Ha, J.-W.} and \textsc{Song, H.~O.} (2022).
\newblock Dataset condensation via efficient synthetic-data parameterization.
\newblock In \textit{International Conference on Machine Learning}. PMLR.

\bibitem[{Kipf and Welling(2016)}]{kipf2016semi}
\textsc{Kipf, T.~N.} and \textsc{Welling, M.} (2016).
\newblock Semi-supervised classification with graph convolutional networks.
\newblock In \textit{International Conference on Learning Representations}.

\bibitem[{Klicpera et~al.(2019)Klicpera, Wei{\ss}enberger and G{\"u}nnemann}]{klicpera2019diffusion}
\textsc{Klicpera, J.}, \textsc{Wei{\ss}enberger, S.} and \textsc{G{\"u}nnemann, S.} (2019).
\newblock Diffusion improves graph learning.
\newblock In \textit{Advances in neural information processing systems}. PMLR.

\bibitem[{Kothawade et~al.(2022)Kothawade, Kaushal, Ramakrishnan, Bilmes and Iyer}]{kothawade2022prism}
\textsc{Kothawade, S.}, \textsc{Kaushal, V.}, \textsc{Ramakrishnan, G.}, \textsc{Bilmes, J.} and \textsc{Iyer, R.} (2022).
\newblock Prism: A rich class of parameterized submodular information measures for guided data subset selection.
\newblock In \textit{Proceedings of the AAAI Conference on Artificial Intelligence}, vol.~36.

\bibitem[{Larsen et~al.(1996)Larsen, Hansen, Svarer and Ohlsson}]{larsen1996design}
\textsc{Larsen, J.}, \textsc{Hansen, L.~K.}, \textsc{Svarer, C.} and \textsc{Ohlsson, M.} (1996).
\newblock Design and regularization of neural networks: the optimal use of a validation set.
\newblock In \textit{Neural Networks for Signal Processing VI. Proceedings of the 1996 IEEE Signal Processing Society Workshop}. IEEE.

\bibitem[{Li and Talwalkar(2020)}]{li2020random}
\textsc{Li, L.} and \textsc{Talwalkar, A.} (2020).
\newblock Random search and reproducibility for neural architecture search.
\newblock In \textit{Uncertainty in artificial intelligence}. PMLR.

\bibitem[{Liu et~al.(2018)Liu, Simonyan and Yang}]{liu2018darts}
\textsc{Liu, H.}, \textsc{Simonyan, K.} and \textsc{Yang, Y.} (2018).
\newblock Darts: Differentiable architecture search.
\newblock In \textit{International Conference on Learning Representations}.

\bibitem[{Liu et~al.(2022)Liu, Li, Chen and Song}]{liu2022graph}
\textsc{Liu, M.}, \textsc{Li, S.}, \textsc{Chen, X.} and \textsc{Song, L.} (2022).
\newblock Graph condensation via receptive field distribution matching.
\newblock \textit{arXiv preprint arXiv:2206.13697} .

\bibitem[{Lorraine et~al.(2020)Lorraine, Vicol and Duvenaud}]{lorraine2020optimizing}
\textsc{Lorraine, J.}, \textsc{Vicol, P.} and \textsc{Duvenaud, D.} (2020).
\newblock Optimizing millions of hyperparameters by implicit differentiation.
\newblock In \textit{International Conference on Artificial Intelligence and Statistics}. PMLR.

\bibitem[{Loukas(2019)}]{loukas2019graph}
\textsc{Loukas, A.} (2019).
\newblock Graph reduction with spectral and cut guarantees.
\newblock \textit{J. Mach. Learn. Res.} \textbf{20} 1--42.

\bibitem[{Loukas and Vandergheynst(2018)}]{loukas2018spectrally}
\textsc{Loukas, A.} and \textsc{Vandergheynst, P.} (2018).
\newblock Spectrally approximating large graphs with smaller graphs.
\newblock In \textit{International Conference on Machine Learning}. PMLR.

\bibitem[{Luketina et~al.(2016)Luketina, Berglund, Greff and Raiko}]{luketina2016scalable}
\textsc{Luketina, J.}, \textsc{Berglund, M.}, \textsc{Greff, K.} and \textsc{Raiko, T.} (2016).
\newblock Scalable gradient-based tuning of continuous regularization hyperparameters.
\newblock In \textit{International conference on machine learning}. PMLR.

\bibitem[{Maddison et~al.(2017)Maddison, Mnih and Teh}]{maddison2016concrete}
\textsc{Maddison, C.~J.}, \textsc{Mnih, A.} and \textsc{Teh, Y.~W.} (2017).
\newblock The concrete distribution: A continuous relaxation of discrete random variables.
\newblock In \textit{International Conference on Learning Representations}.

\bibitem[{Martens and Grosse(2015)}]{martens2015optimizing}
\textsc{Martens, J.} and \textsc{Grosse, R.} (2015).
\newblock Optimizing neural networks with kronecker-factored approximate curvature.
\newblock In \textit{International conference on machine learning}. PMLR.

\bibitem[{Nguyen et~al.(2020)Nguyen, Chen and Lee}]{nguyen2020dataset}
\textsc{Nguyen, T.}, \textsc{Chen, Z.} and \textsc{Lee, J.} (2020).
\newblock Dataset meta-learning from kernel ridge-regression.
\newblock In \textit{International Conference on Learning Representations}.

\bibitem[{Nguyen et~al.(2021)Nguyen, Novak, Xiao and Lee}]{nguyen2021dataset}
\textsc{Nguyen, T.}, \textsc{Novak, R.}, \textsc{Xiao, L.} and \textsc{Lee, J.} (2021).
\newblock Dataset distillation with infinitely wide convolutional networks.
\newblock \textit{Advances in Neural Information Processing Systems} \textbf{34} 5186--5198.

\bibitem[{Ochs et~al.(2015)Ochs, Ranftl, Brox and Pock}]{ochs2015bilevel}
\textsc{Ochs, P.}, \textsc{Ranftl, R.}, \textsc{Brox, T.} and \textsc{Pock, T.} (2015).
\newblock Bilevel optimization with nonsmooth lower level problems.
\newblock In \textit{International Conference on Scale Space and Variational Methods in Computer Vision}. Springer.

\bibitem[{Paul et~al.(2021)Paul, Ganguli and Dziugaite}]{paul2021deep}
\textsc{Paul, M.}, \textsc{Ganguli, S.} and \textsc{Dziugaite, G.~K.} (2021).
\newblock Deep learning on a data diet: Finding important examples early in training.
\newblock \textit{Advances in Neural Information Processing Systems} \textbf{34} 20596--20607.

\bibitem[{Pedregosa(2016)}]{pedregosa2016hyperparameter}
\textsc{Pedregosa, F.} (2016).
\newblock Hyperparameter optimization with approximate gradient.
\newblock In \textit{International conference on machine learning}. PMLR.

\bibitem[{Puny et~al.(2020)Puny, Ben-Hamu and Lipman}]{puny2020graph}
\textsc{Puny, O.}, \textsc{Ben-Hamu, H.} and \textsc{Lipman, Y.} (2020).
\newblock From graph low-rank global attention to 2-fwl approximation.
\newblock In \textit{International Conference on Machine Learning}. PMLR.

\bibitem[{Rebuffi et~al.(2017)Rebuffi, Kolesnikov, Sperl and Lampert}]{rebuffi2017icarl}
\textsc{Rebuffi, S.-A.}, \textsc{Kolesnikov, A.}, \textsc{Sperl, G.} and \textsc{Lampert, C.~H.} (2017).
\newblock icarl: Incremental classifier and representation learning.
\newblock In \textit{Proceedings of the IEEE conference on Computer Vision and Pattern Recognition}.

\bibitem[{Rong et~al.(2020)Rong, Bian, Xu, Xie, Wei, Huang and Huang}]{rong2020self}
\textsc{Rong, Y.}, \textsc{Bian, Y.}, \textsc{Xu, T.}, \textsc{Xie, W.}, \textsc{Wei, Y.}, \textsc{Huang, W.} and \textsc{Huang, J.} (2020).
\newblock Self-supervised graph transformer on large-scale molecular data.
\newblock In \textit{Advances in neural information processing systems}, vol.~33.

\bibitem[{Satuluri et~al.(2011)Satuluri, Parthasarathy and Ruan}]{satuluri2011local}
\textsc{Satuluri, V.}, \textsc{Parthasarathy, S.} and \textsc{Ruan, Y.} (2011).
\newblock Local graph sparsification for scalable clustering.
\newblock In \textit{Proceedings of the 2011 ACM SIGMOD International Conference on Management of data}.

\bibitem[{Sener and Savarese(2018)}]{sener2018active}
\textsc{Sener, O.} and \textsc{Savarese, S.} (2018).
\newblock Active learning for convolutional neural networks: A core-set approach.
\newblock In \textit{International Conference on Learning Representations}.

\bibitem[{Shaban et~al.(2019)Shaban, Cheng, Hatch and Boots}]{shaban2019truncated}
\textsc{Shaban, A.}, \textsc{Cheng, C.-A.}, \textsc{Hatch, N.} and \textsc{Boots, B.} (2019).
\newblock Truncated back-propagation for bilevel optimization.
\newblock In \textit{The 22nd International Conference on Artificial Intelligence and Statistics}. PMLR.

\bibitem[{Simonovsky and Komodakis(2018)}]{simonovsky2018graphvae}
\textsc{Simonovsky, M.} and \textsc{Komodakis, N.} (2018).
\newblock Graphvae: Towards generation of small graphs using variational autoencoders.
\newblock In \textit{International conference on artificial neural networks}. Springer.

\bibitem[{Such et~al.(2020)Such, Rawal, Lehman, Stanley and Clune}]{such2020generative}
\textsc{Such, F.~P.}, \textsc{Rawal, A.}, \textsc{Lehman, J.}, \textsc{Stanley, K.} and \textsc{Clune, J.} (2020).
\newblock Generative teaching networks: Accelerating neural architecture search by learning to generate synthetic training data.
\newblock In \textit{International Conference on Machine Learning}. PMLR.

\bibitem[{Toneva et~al.(2018)Toneva, Sordoni, des Combes, Trischler, Bengio and Gordon}]{toneva2018empirical}
\textsc{Toneva, M.}, \textsc{Sordoni, A.}, \textsc{des Combes, R.~T.}, \textsc{Trischler, A.}, \textsc{Bengio, Y.} and \textsc{Gordon, G.~J.} (2018).
\newblock An empirical study of example forgetting during deep neural network learning.
\newblock In \textit{International Conference on Learning Representations}.

\bibitem[{Veli{\v{c}}kovi{\'c} et~al.(2018)Veli{\v{c}}kovi{\'c}, Cucurull, Casanova, Romero, Lio and Bengio}]{velivckovic2017graph}
\textsc{Veli{\v{c}}kovi{\'c}, P.}, \textsc{Cucurull, G.}, \textsc{Casanova, A.}, \textsc{Romero, A.}, \textsc{Lio, P.} and \textsc{Bengio, Y.} (2018).
\newblock Graph attention networks.
\newblock In \textit{International Conference on Learning Representations}.

\bibitem[{Wang et~al.(2022)Wang, Zhao, Peng, Zhu, Yang, Wang, Huang, Bilen, Wang and You}]{wang2022cafe}
\textsc{Wang, K.}, \textsc{Zhao, B.}, \textsc{Peng, X.}, \textsc{Zhu, Z.}, \textsc{Yang, S.}, \textsc{Wang, S.}, \textsc{Huang, G.}, \textsc{Bilen, H.}, \textsc{Wang, X.} and \textsc{You, Y.} (2022).
\newblock Cafe: Learning to condense dataset by aligning features.
\newblock In \textit{Proceedings of the IEEE/CVF Conference on Computer Vision and Pattern Recognition}.

\bibitem[{Wang et~al.(2020)Wang, Cheng, Chen, Tang and Hsieh}]{wang2020rethinking}
\textsc{Wang, R.}, \textsc{Cheng, M.}, \textsc{Chen, X.}, \textsc{Tang, X.} and \textsc{Hsieh, C.-J.} (2020).
\newblock Rethinking architecture selection in differentiable nas.
\newblock In \textit{International Conference on Learning Representations}.

\bibitem[{Wang et~al.(2018)Wang, Zhu, Torralba and Efros}]{wang2018dataset}
\textsc{Wang, T.}, \textsc{Zhu, J.-Y.}, \textsc{Torralba, A.} and \textsc{Efros, A.~A.} (2018).
\newblock Dataset distillation.
\newblock \textit{arXiv preprint arXiv:1811.10959} .

\bibitem[{Wang et~al.(2019)Wang, Zhang and Ba}]{wang2019solving}
\textsc{Wang, Y.}, \textsc{Zhang, G.} and \textsc{Ba, J.} (2019).
\newblock On solving minimax optimization locally: A follow-the-ridge approach.
\newblock In \textit{International Conference on Learning Representations}.

\bibitem[{Wang et~al.(2021)Wang, Di and Chen}]{wang2021autogel}
\textsc{Wang, Z.}, \textsc{Di, S.} and \textsc{Chen, L.} (2021).
\newblock Autogel: An automated graph neural network with explicit link information.
\newblock \textit{Advances in Neural Information Processing Systems} \textbf{34} 24509--24522.

\bibitem[{Welling(2009)}]{welling2009herding}
\textsc{Welling, M.} (2009).
\newblock Herding dynamical weights to learn.
\newblock In \textit{Proceedings of the 26th Annual International Conference on Machine Learning}.

\bibitem[{Williams(1992)}]{williams1992simple}
\textsc{Williams, R.~J.} (1992).
\newblock Simple statistical gradient-following algorithms for connectionist reinforcement learning.
\newblock \textit{Reinforcement learning}  5--32.

\bibitem[{Xie et~al.(2018)Xie, Zheng, Liu and Lin}]{xie2018snas}
\textsc{Xie, S.}, \textsc{Zheng, H.}, \textsc{Liu, C.} and \textsc{Lin, L.} (2018).
\newblock Snas: stochastic neural architecture search.
\newblock In \textit{International Conference on Learning Representations}.

\bibitem[{Xu et~al.(2018)Xu, Hu, Leskovec and Jegelka}]{xu2018powerful}
\textsc{Xu, K.}, \textsc{Hu, W.}, \textsc{Leskovec, J.} and \textsc{Jegelka, S.} (2018).
\newblock How powerful are graph neural networks?
\newblock In \textit{International Conference on Learning Representations}.

\bibitem[{Yang et~al.(2016)Yang, Cohen and Salakhudinov}]{yang2016revisiting}
\textsc{Yang, Z.}, \textsc{Cohen, W.} and \textsc{Salakhudinov, R.} (2016).
\newblock Revisiting semi-supervised learning with graph embeddings.
\newblock In \textit{International conference on machine learning}. PMLR.

\bibitem[{Zela et~al.(2019)Zela, Elsken, Saikia, Marrakchi, Brox and Hutter}]{zela2019understanding}
\textsc{Zela, A.}, \textsc{Elsken, T.}, \textsc{Saikia, T.}, \textsc{Marrakchi, Y.}, \textsc{Brox, T.} and \textsc{Hutter, F.} (2019).
\newblock Understanding and robustifying differentiable architecture search.
\newblock In \textit{International Conference on Learning Representations}.

\bibitem[{Zeng et~al.(2019)Zeng, Zhou, Srivastava, Kannan and Prasanna}]{zeng2019graphsaint}
\textsc{Zeng, H.}, \textsc{Zhou, H.}, \textsc{Srivastava, A.}, \textsc{Kannan, R.} and \textsc{Prasanna, V.} (2019).
\newblock Graphsaint: Graph sampling based inductive learning method.
\newblock In \textit{International Conference on Learning Representations}.

\bibitem[{Zhang et~al.(2020)Zhang, Zhang, Xia and Sun}]{zhang2020graph}
\textsc{Zhang, J.}, \textsc{Zhang, H.}, \textsc{Xia, C.} and \textsc{Sun, L.} (2020).
\newblock Graph-bert: Only attention is needed for learning graph representations.
\newblock \textit{arXiv preprint arXiv:2001.05140} .

\bibitem[{Zhao and Bilen(2021{\natexlab{a}})}]{zhao2021dataseta}
\textsc{Zhao, B.} and \textsc{Bilen, H.} (2021{\natexlab{a}}).
\newblock Dataset condensation with differentiable siamese augmentation.
\newblock In \textit{International Conference on Machine Learning}. PMLR.

\bibitem[{Zhao and Bilen(2021{\natexlab{b}})}]{zhao2021datasetb}
\textsc{Zhao, B.} and \textsc{Bilen, H.} (2021{\natexlab{b}}).
\newblock Dataset condensation with distribution matching.
\newblock \textit{arXiv preprint arXiv:2110.04181} .

\bibitem[{Zhao et~al.(2020)Zhao, Mopuri and Bilen}]{zhao2020dataset}
\textsc{Zhao, B.}, \textsc{Mopuri, K.~R.} and \textsc{Bilen, H.} (2020).
\newblock Dataset condensation with gradient matching.
\newblock In \textit{International Conference on Learning Representations}.

\bibitem[{Zoph et~al.(2018)Zoph, Vasudevan, Shlens and Le}]{zoph2018learning}
\textsc{Zoph, B.}, \textsc{Vasudevan, V.}, \textsc{Shlens, J.} and \textsc{Le, Q.~V.} (2018).
\newblock Learning transferable architectures for scalable image recognition.
\newblock In \textit{Proceedings of the IEEE conference on computer vision and pattern recognition}.

\end{thebibliography}
\bibliographystyle{ims}

\end{document}